\documentclass[11pt,letterpaper]{article}

\usepackage[top=0.7in,bottom=0.7in,left=0.7in,right=0.7in]{geometry}

\usepackage{natbib}

\usepackage{microtype}
\usepackage{graphicx}
\usepackage{booktabs} 
\usepackage{subcaption}

\usepackage{hyperref}

\usepackage{amsmath}
\usepackage{amssymb}
\usepackage{mathtools}
\usepackage{amsthm}

\usepackage{mathabx}
\usepackage{bm}
\usepackage{graphicx}
\usepackage{xcolor}
\usepackage{tikz}

\newcommand{\real}{\mathbb{R}}
\newcommand{\T}{\intercal}

\newcommand{\seq}[1]{[#1]}

\newcommand{\sign}{{\rm{sign}}}
\newcommand{\prob}{\mathbb{P}}
\newcommand{\E}{\mathbb{E}}
\newcommand{\calO}{\mathcal{O}}

\newcommand{\sym}{\mathbb{S}^d_+}

\newcommand{\be}{\mathbf{e}}

\newcommand{\bH}{H}

\newcommand{\bHhat}{\widehat{\bH}}
\newcommand{\bzero}{\mathbf{0}}

\newcommand{\inner}[2]{\langle #1, #2 \rangle}
\newcommand{\diag}{{\rm{diag}}}
\newcommand{\eig}{{\rm{eig}}}
\newcommand{\diff}[2]{\frac{\partial{#1}}{\partial{#2}}}
\newcommand{\inertia}{{\rm{In}}}
\newcommand{\tr}{{\rm{trace}}}
\newcommand{\calI}{\mathcal{I}}

\usepackage[capitalize,noabbrev]{cleveref}

\theoremstyle{plain}
\newtheorem{theorem}{Theorem}[section]

\newtheorem{lemma}[theorem]{Lemma}

\theoremstyle{definition}
\newtheorem{definition}[theorem]{Definition}
\newtheorem{assumption}[theorem]{Assumption}
\theoremstyle{remark}

\usepackage[textsize=tiny]{todonotes}

\title{Sparse Mixed Linear Regression with Guarantees: Taming an Intractable Problem with Invex Relaxation}
\date{}

\author{%
	Adarsh Barik \\
	Department of Computer Science\\
	Purdue University\\
	West Lafayette, Indiana, USA\\
	\texttt{abarik@purdue.edu} \\
	\and
	Jean Honorio \\
	Department of Computer Science \\
	Purdue University \\
	West Lafayette, Indiana, USA\\
	\texttt{jhonorio@purdue.edu} \\
}

\begin{document}
\maketitle

\begin{abstract}
In this paper, we study the problem of sparse mixed linear regression on an unlabeled dataset that is generated from linear measurements from two different regression parameter vectors. Since the data is unlabeled, our task is not only to figure out a good approximation of the regression parameter vectors but also to label the dataset correctly. In its original form, this problem is NP-hard. The most popular algorithms to solve this problem (such as Expectation-Maximization) have a tendency to stuck at local minima. We provide a novel invex relaxation for this intractable problem which leads to a solution with provable theoretical guarantees. This relaxation enables exact recovery of data labels. Furthermore, we recover a close approximation of the regression parameter vectors which match the true parameter vectors in support and sign. Our formulation uses a carefully constructed primal dual witnesses framework for the invex problem. Furthermore, we show that the sample complexity of our method is only logarithmic in terms of the dimension of the regression parameter vectors.
\end{abstract}

\section{Introduction}
\label{sec:introduction}

In this paper, we study sparse mixed linear regression where the measurements come from one of the two regression models depending upon the unknown label $z_i^* \in \{0, 1\}$. The observation model can be described as follows:
\begin{align}
\label{eq:mixed linear regression}
\begin{split}
	y_i = z^*_i \inner{X_i}{\beta_1^*} + (1 - z^*_i) \inner{X_i}{\beta_2^*} + e_i, \forall i \in \{ 1,\cdots, n \} \; ,
\end{split}
\end{align}
where $X_i \in \real^d, y_i \in \real$ and $e_i \in \real$ is independent additive noise. The regression parameter vectors $\beta_1^* \in \real^d$, $\beta_2^* \in \real^d$ are sparse vectors with possibly non-overlapping supports. 

Mixed linear regression models have been extensively used in a wide range of applications~\citep{grun2007applications} which include but are not limited to behavioral health-care~\citep{deb2000estimates}, market segmentation~\citep{wedel2000market}, music perception studies~\citep{viele2002modeling} and vehicle merging~\citep{li2019modeling}. The main task of the problem is to estimate the regression parameter vectors and the unknown labels accurately from linear measurements. However, the problem is NP-hard without any assumptions~\citep{yi2014alternating}. Being such a difficult problem, it also lends itself to be used as a benchmark for many non-convex optimization algorithms~\citep{chaganty2013spectral,klusowski2019estimating}.

\paragraph{Related Work.} 
There have been many approaches to solve the mixed linear regression problem after it was introduced by~\cite{wedel1995mixture}. The most popular and natural approach has been to use Expectation-minimization (EM) based alternate minimization algorithms (see \citet{ghosh2020alternating} and references therein). More broadly, the problem can be modeled under the hierarchical mixtures of experts model~\citep{jordan1994hierarchical} and solved using EM based algorithms. All these methods run the risk of getting stuck at local minima~\citep{wu1983convergence} without good initialization. \cite{yi2014alternating} provides a good initialization for the noiseless case under strict technical conditions, however their method does not provide any guarantees for the noisy case. Based on the recent work of~\cite{anandkumar2014tensor} and~\cite{hsu2013learning},~\citep{chaganty2013spectral} have proposed an approach which uses a third order moment method based on tensor decomposition. Their approach suffers from high sample complexity (up to $\calO(d^6)$) due to tensor decomposition. \cite{stadler2010l} proposed an $\ell_1$-regularized approach for the sparse case and showed the existence of a local minimizer with correct support but there are no guarantees that EM achieves this local minima. \cite{chen2014convex} provided a convex relaxation involving nuclear norms for the problem. They do not focus on providing guarantees for exact label recovery and their results only hold for bounded noise  and require balanced samples (almost equal number of samples for both labels). Besides, the optimization problems involving nuclear norms are computationally heavy and slow. The mixed linear regression problem can also be modeled as a subspace clustering problem. But typically these problems require $\calO(d^2)$ measurements to have a unique solution~\citep{vidal2005generalized,elhamifar2013sparse}.

\paragraph{Contribution.} 

Broadly, we can categorize our contribution in the following points:
\begin{itemize}
\item \textbf{A Combinatorial Problem}: We view the problem as a combinatorial version of a mixture of sparse linear regressions. The exact label recovery is as important for us as the recovery of regression vectors. This added exact label recovery guarantee comes at no extra cost in terms of the performance.
\item \textbf{Invex Relaxation}: We solve a non-convex problem which is known to be intractable. We propose a novel relaxation of the combinatorial problem and formally show that this relaxation is invex. 
\item \textbf{Theoretical Guarantees}:  Our method solves two sparse linear regressions and a label recovery problem simultaneously with theoretical guarantees. To that end, we recover the true labels and sparse regression parameter vectors which are correct up to the sign of entries with respect to the true parameter vectors. As a side product, we propose a novel primal-dual witness construction for our invex problem and provide theoretical guarantees for recovery. The sample complexity of our method only varies logarithmically with respect to dimension of the regression parameter vector.
\item \textbf{A Novel Framework}: It should be noted that we are providing a novel framework (not an algorithm) to solve the problem. This opens the door for many algorithms to be used for this problem.         
\end{itemize}    

\section{Problem Setup}
\label{sec:problem setup}

In this section, we collect the notations used throughout the paper and define our problem formally. We consider a problem where measurements come from a mixture of two linear regression problem. Let $y_i \in \real $ be the response variable and $X_i \in \real^d$ be the observed attributes. Let $z^*_i \in \{0, 1\}$ denote the unknown label associated with measurement $i$. The response $y_i$ is generated using the observation model \eqref{eq:mixed linear regression} where $e_i \in \real$ is an independent noise term. We collect a total of $n$ linear measurements with $n_1$ measurements belonging to label $1$ and $n_2$ measurements belonging to label $0$. Clearly, $n = n_1 + n_2$. We take $\| \beta_1^* \|_1 \leq b_1$ and $\| \beta_2^* \|_1 \leq b_2$.  

Let $\seq{d}$ denote the set $\{1, 2, \cdots, d\}$. We assume $X_i \in \real^d$ to be a zero mean sub-Gaussian random vector~\citep{hsu2012tail} with covariance $\Sigma \in \sym$, i.e., there exists a $\rho > 0$, such that for all $\tau \in \real^d$ the following holds: $	\E(\exp(\tau^\T X_i)) \leq \exp(\frac{\| \tau \|_2^2 \rho^2}{2}) $. By simply taking $\tau_j = r$ and $\tau_k = 0, \forall k \ne j$, it follows that each entry of $X_i$ is sub-Gaussian with parameter $\rho$. In particular, we will assume that $\forall j \in \seq{d}\, , \frac{X_{ij}}{\sqrt{\Sigma_{jj}}}$ is a sub-Gaussian random variable with parameter $\sigma > 0$. It follows trivially that $\max_{j \in \seq{d}} \sqrt{\Sigma_{jj}} \sigma \leq \rho$.  We will further assume that $e_i$ is zero mean independent sub-Gaussian noise with variance $\sigma_e$. Our setting works with a variety of random variables as the class of sub-Gaussian random variable includes for instance Gaussian variables, any bounded random variable (e.g., Bernoulli, multinomial, uniform), any random variable with strictly log-concave density, and any finite mixture of sub-Gaussian variables. 

The parameter vectors $\beta_1^* \in \real^d$ and $\beta_2^* \in \real^d$ are $s_1$-sparse and $s_2$-sparse respectively, i.e., at most $s_1$ entries of $\beta_1^*$ are non-zero whereas at most $s_2$ entries of $\beta_2^*$ are non-zero. We receive $n$ i.i.d. samples of $X_i \in \real^d$ and $y_i \in \real$ and collect them in $X \in \real^{n \times d}$  and $y \in \real^n$ respectively. Similarly, $z^* \in \{0, 1 \}^n$ collects all the labels. Our goal is to recover $\beta_1^*, \beta_2^*$ and $z^*$ using the samples $(X, y)$. 

We denote a matrix $A \in \real^{p \times q}$ restricted to the columns and rows in $P \subseteq \seq{p}$ and $Q \subseteq \seq{q}$ respectively as $A_{PQ}$. Similarly, a vector $v \in \real^p$ restricted to entries in $P$ is denoted as $v_P$. We use $\eig_i(A)$ to denote the $i$-th eigenvalue ($1$st being the smallest) of matrix $A$. Similarly, $\eig_{\max}(A)$ denotes the maximum eigenvalue of matrix $A$. We use $\diag(A)$ to denote a vector containing the diagonal element of matrix $A$. By overriding the same notation, we use $\diag(v)$ to denote a diagonal matrix with its diagonal being the entries in vector $v$. We denote the inner product between two matrices (or vectors) $A$ and $B$ by $\inner{A}{B}$, i.e., $\inner{A}{B} = \tr(A^\T B)$, where $\tr$ denotes the trace of a matrix. The notation $A \succeq B$ denotes that $A - B$ is a positive semidefinite matrix. Similarly, $A \succ B$ denotes that $A-B$ is a positive definite matrix. For vectors, $\| v \|_p$ denotes the $\ell_p$-vector norm of vector $v \in \real^d$, i.e., $\| v \|_p = ( \sum_{i=1}^d |v_i|^p)^{\frac{1}{p}}$. If $p = \infty$, then we define $\| v \|_{\infty} = \max_{i=1}^d |v_i|$. As is the tradition, we used $\| v \|_0$ to denote number of non-zero entries on vector $v$. It should be remembered that $\ell_0$ is not a proper vector norm. For matrices, $\| A \|_p$ denotes the induced $\ell_p$-matrix norm for matrix $A \in \real^{p \times q}$. In particular, $\| A \|_2$ denotes the spectral norm of $A$ and $\| A \|_{\infty} \triangleq \max_{i \in \seq{p}} \sum_{j=1}^q |A_{ij}|$. For a matrix $A \in \real^{p \times q}$,  $A(:) \in \real^{pq}$ denotes a vector which collects all entries of the matrix $A$. We define an operator $\sign(A)$ for a matrix(or vector) $A$, which returns a matrix (or a vector) with entries being the sign of the entries of $A$. A function $f(x)$ is of order $\Omega(g(x))$ and denoted by $ f(x) = \Omega(g(x))$, if there exists a constant $C > 0$ such that for big enough $x_0$, $f(x) \geq C g(x), \forall x \geq x_0$. Similarly, a function $f(x)$ is of order $\calO(g(x))$ and denoted by $ f(x) = \calO(g(x))$, if there exists a constant $C > 0$ such that for big enough $x_0$, $f(x) \leq C g(x), \forall x \geq x_0$. For brevity in our notations, we treat any quantity independent of $d, s$ and $n$ as constant. Detailed proofs for lemmas and theorems are available in the supplementary material.

\section{A Novel Invex Relaxation}
\label{sec:a novel invex relaxation}

In this section, we introduce a combinatorial formulation for mixed linear regression (MLR) and propose a novel invex relaxation for this problem. Since the measurements come from a true observation model \eqref{eq:mixed linear regression}, we can write the following optimization problem to estimate $\beta_1^*, \beta_2^*$ and $z^*$. 
\begin{definition}[Standard MLR]
	\label{def:standard mlr}
	\begin{align}
	\label{eq:opt prob standard mlr}
	\begin{split}
	\begin{matrix}
	\min\limits_{\beta_1 \in \real^d, \beta_2 \in \real^d, z \in \{0, 1\}^n} &  l(z, \beta_1, \beta_2) \\
	\text{such that} & \| \beta_1 \|_0 = s_1, \; \| \beta_2 \|_0 = s_2
	\end{matrix} 
	\end{split}
	\end{align}
	where $l(z, \beta_1, \beta_2) = \frac{1}{n}  \sum_{i=1}^n z_i (y_i - X_i^\T \beta_1)^2 + (1 - z_i) (y_i - X_i^\T \beta_2)^2$.
\end{definition}

Even without constraints, optimization problem \eqref{eq:opt prob standard mlr} is a non-convex NP-hard problem~\citep{yi2014alternating} in its current form. In fact, a continuous relaxation of $z \in [0, 1]^n$ does not help and it still remains a non-convex problem (See Appendix \ref{sec:std mlr nonconvex}). Furthermore, the sparsity constraints make it even difficult to solve. To deal with this intractability, we come up with a novel invex relaxation of the problem.

For ease of notation, we define the following quantities:
\begin{align}
\label{eq:Si and W}
\begin{split}
	S_i &= \begin{bmatrix} X_i \\ -y_i \end{bmatrix} \begin{bmatrix} X_i^\T & -y_i \end{bmatrix} = \begin{bmatrix} X_i X_i^\T & -  X_i y_i\\ -y_i X_i^\T & y_i^2 \end{bmatrix}  ,
\end{split}
\end{align}
We provide the following invex relaxation to the optimization problem \eqref{eq:opt prob standard mlr}.
\begin{definition}[Invex MLR]
	\label{def:invex mlr}
	\begin{align}
	\label{eq:opt prob invex mlr}
	\begin{split}
	\begin{matrix}
	\min\limits_{t, W, U} & f(t, W, U) + \lambda_1 g(t, W, U) + \lambda_2 h(t, W, U) \\
	\text{such that} & \\
	& W \succeq \mathbf{0}, \; U \succeq \mathbf{0} \\
	& W_{d+1, d+1} = 1, \;  U_{d+1, d+1} = 1 \\
	& \| t \|_{\infty} \leq 1 
	\end{matrix} \, .
	\end{split}
	\end{align}
	where $ f(t, W, U) = \sum_{i=1}^n \frac{1}{2} \inner{S_i}{W + U} + \sum_{i=1}^n \frac{1}{2} t_i \inner{S_i}{W - U}, g(t, W, U) = \|W(:)\|_1$ and $h(t, W, U) = \| U(:)\|_1$ and $\lambda_1$ and $\lambda_2$ are positive regularizers.
\end{definition}

To get an intuition behind this formulation, one can think of $W$ and $U$ as two rank-1 matrices which are defined as follows:
\begin{align}
\begin{split}
		W &= \begin{bmatrix} \beta_1 \\ 1 \end{bmatrix} \begin{bmatrix} \beta_1^\T & 1 \end{bmatrix}, \; U = \begin{bmatrix} \beta_2 \\ 1 \end{bmatrix} \begin{bmatrix} \beta_2^\T & 1 \end{bmatrix} 
\end{split}
\end{align}

The variable $t$ is simply a replacement of variable $z$, i.e., $z_i = \frac{t_i + 1}{2}$. Then after substituting $t, W$ and $U$ in $f(t, W, U)$, we get back $l(z, \beta_1, \beta_2)$. The $\ell_1$-regularization of $W(:)$ and $U(:)$ helps us ensure sparsity.
Note that for fixed $t$, optimization problem \eqref{eq:opt prob invex mlr} is continuous and convex with respect to $W$ and $U$. Specifically, it merges two independent regularized semidefinite programs. Unfortunately, problem \eqref{eq:opt prob invex mlr} is not jointly convex on $t, W$ and $U$, and thus, it might still remain difficult to solve. Next, we will provide arguments that despite being non-convex, optimization problem \eqref{eq:opt prob invex mlr} belongs to a particular class of non-convex functions namely ``invex'' functions. The ``invexity'' of functions can be defined as a generalization of convexity~\citep{hanson1981sufficiency}. Invexity has been recently used by~\cite{barik2021fair} to solve fair sparse regression problem with clustering. While, we borrow some definitions from their work to suit our context, we should emphasize that our problem is fundamentally different than their problem. They use two groups in sparse regression which have different means and they try to achieve fairness. While here, we have two groups with the same mean and there is no unfairness in the problem. We also model our parameter vectors with positive semidefinite matrices which is fundamentally different from their approach. 

\begin{definition}[Invex function~\citep{barik2021fair}]
	\label{def:invex function}
	Let $\phi(t)$ be a function defined on a set $C$. Let $\eta$ be a vector valued function defined in $C \times C$ such that $\eta(t_1, t_2)^\T \nabla \phi(t_2)$, is well defined $\forall t_1, t_2 \in C$. Then, $\phi(t)$ is a $\eta$-invex function if $\phi(t_1) - \phi(t_2) \geq \eta(t_1, t_2)^\T \nabla \phi(t_2), \, \forall t_1, t_2 \in C$.
\end{definition}

Note that convex functions are $\eta$-invex for $\eta(t_1,t_2) = t_1 - t_2$. \cite{hanson1981sufficiency} showed that if the objective function and constraints are both $\eta$-invex with respect to same $\eta$ defined in $C \times C$, then Karush-Kuhn-Tucker (KKT) conditions are sufficient for optimality, while it is well-known that KKT conditions are necessary. \cite{ben1986invexity} showed a function is invex if and only if each of its stationarity point is a global minimum. 

In the next lemma, we show that the relaxed optimization problem \eqref{eq:opt prob invex mlr} is indeed $\eta$-invex for a particular $\eta$ defined in $C \times C$ and a well defined set $C$. Let $C = \{ (t, W, U) \mid t \in [-1, 1], W \succeq \mathbf{0}, U \succeq \mathbf{0}, W_{d+1, d+1} = 1, U_{d+1, d+1} = 1 \}$.

\begin{lemma}
	\label{lem:invexity}
	For $(t, W, U) \in C$, the functions $ f(t, W, U) = \sum_{i=1}^n \frac{1}{2} \inner{S_i}{W + U} + \sum_{i=1}^n \frac{1}{2} t_i \inner{S_i}{W - U}, g(t, W, U) = \|W(:)\|_1$ and $h(t, W, U) = \| U(:)\|_1$ are $\eta$-invex for $\eta(t, \bar{t}, W, \overline{W}, U, \overline{U}) \triangleq \begin{bmatrix} \eta_t \\ \eta_{W} \\ \eta_{U}\end{bmatrix}$, where $\eta_t = \mathbf{0} \in \real^n, \eta_{W} = -\overline{W}$ and $\eta_{U} = - \overline{U}$. We abuse the vector/matrix notation (by ignoring the dimensions) for clarity of presentation, and avoid the vectorization of matrices.
\end{lemma}
Now that we have established that optimization problem \eqref{eq:opt prob invex mlr} is invex, we are ready to discuss our main results in the next section.

\section{Main Results}
\label{sec:main results}

In this section, we present our main results along with the technical assumptions. Our main goal is to show that the solution to optimization problem \eqref{eq:opt prob invex mlr} recovers the labels $t^*$ exactly and also recovers a good approximation of $\beta_1^*$ and $\beta_2^*$. In that, we will show that the recovered $\beta_1$ and $\beta_2$ have the same support and sign as $\beta_1^*$ and $\beta_2^*$ respectively and are close to the true vectors in $\ell_2$-norm. But before that, we will describe a set of technical assumptions which will help us in our analysis.

\subsection{Assumptions}
\label{subsec:assumptions}

Our first assumption ensures that each sample can be assigned only one label. Formally,
\begin{assumption}[Identifiability]
	\label{assum:identifiability}
	For $i \in \seq{n},\; - \frac{1}{2}(y_i - X_i^\T \beta_1^*)^2 + \frac{1}{2} (y_i - X_i^\T \beta_2^*)^2 \geq \epsilon$ if $z_i^* = 1$ and $ \frac{1}{2}(y_i - X_i^\T \beta_1^*)^2 - \frac{1}{2}(y_i - X_i^\T \beta_2^*)^2 \geq \epsilon$ if $z_i^* = 0$ for some $\epsilon > 0$.
\end{assumption}
Clearly, if Assumption \ref{assum:identifiability} does not hold for sample $i$, then we can reverse the label of sample $i$ without increasing objective function of optimization problem~\eqref{eq:opt prob standard mlr}. Another equivalent way of expressing Assumption \ref{assum:identifiability} is as following: for $i \in \seq{n},\; \inner{S_i}{W^*} < \inner{S_i}{U^*}$ if $z_i^* = 1$ and $\inner{S_i}{W^*} > \inner{S_i}{U^*}$ if $z_i^* = 0$ where,
\begin{align}
\begin{split}
W^* &= \begin{bmatrix} \beta_1^* \\ 1 \end{bmatrix} \begin{bmatrix} {\beta_1^*}^\T & 1 \end{bmatrix}, \; U^* = \begin{bmatrix} \beta_2^* \\ 1 \end{bmatrix} \begin{bmatrix} {\beta_2^*}^\T & 1 \end{bmatrix} \; .
\end{split}
\end{align}

Let $P$ denote the support of $\beta_1^*$, i.e., $P = \{ i\, | \, {\beta_1^*}_i \ne 0, \, i \in \seq{d} \}$ and let $Q$ denote the support of $\beta_2^*$, i.e., $Q = \{ i\, | \, {\beta_2^*}_i \ne 0, \, i \in \seq{d} \}$. Similarly, we define their complement as $P^c = \{ i\, | \, {\beta_1^*}_i = 0, \, i \in \seq{d} \}$ and $Q^c = \{ i\, | \, {\beta_2^*}_i = 0, \, i \in \seq{d} \}$ . We take $|P| = s_1, |P^c| = d - s_1, |Q| = s_2$ and $|Q^c| = d - s_2$. For ease of notation, we define $H \triangleq \E(X_i X_i^\T) \forall i \in \seq{n}$. Let $\calI_1 \triangleq \{ i \, | z_i^* = 1, i \in \seq{n} \}$ and $\calI_2 \triangleq \{ i \, | z_i^* = 0, i \in \seq{n} \}$. We define $\widehat{H}_1 \triangleq \frac{1}{n_1} \sum_{i \in \calI_1} X_i X_i^\T$ and $\widehat{H}_2 \triangleq \frac{1}{n_2} \sum_{i \in \calI_2} X_i X_i^\T$. As our next assumption, we need the minimum eigenvalue of the population covariance matrix of $X$ restricted to rows and columns in $P$ (similarly in $Q$) to be greater than zero. 
\begin{assumption}[Positive Definiteness of Hessian]
	\label{assum:postive definite}
	$H_{PP} \succ \mathbf{0} $ and $H_{QQ} \succ \mathbf{0}$ or equivalently 
	\begin{align*}
	\min( \eig_{\min}(H_{PP}), \eig_{\min}(H_{QQ}) ) = C_{\min} > 0 .
\end{align*}
 We also assume that $\eig_{\max}(H) = C_{\max} > 0$. Note that $\max( \eig_{\max}(H_{PP}), \eig_{\max}(H_{QQ}) ) \leq C_{\max}$.  
\end{assumption}

In practice, we only deal with finite samples and not populations. In the next lemma, we will show that with a sufficient number of samples, a condition similar to Assumption \ref{assum:postive definite} holds with high probability in the finite-sample setting.
\begin{lemma}
	\label{lem:sample positive definite}
If Assumption \ref{assum:postive definite} holds and $n_1 = \Omega(\frac{s_1 + \log d}{C_{\min}^2})$ and $n_2 = \Omega(\frac{s_2 + \log d}{C_{\min}^2})$ , then 
\begin{align*}
\min( \eig_{\min}(\widehat{H}_{1_{PP}}), \eig_{\min}(\widehat{H}_{2_{QQ}}) ) \geq \frac{C_{\min}}{2}
\end{align*}
 and 
 \begin{align*}
 	\max( \eig_{\max}(\widehat{H}_{1_{PP}}), \eig_{\max}(\widehat{H}_{2_{QQ}}) ) \leq \frac{3C_{\max}}{2}
 \end{align*}
 with probability at least $1 - \calO(\frac{1}{d})$.
\end{lemma}

As the third assumption, we will need to ensure that the variates outside the support of $\beta_1^*$ and $\beta_2^*$ do not exert lot of influence on the variates in the support of $\beta_1^*$ and $\beta_2^*$ respectively. For this, we use a technical condition commonly known as the mutual incoherence condition. It has been previously used in many problems related to regularized regression such as compressed sensing~\citep{wainwright2009sharp}, Markov random fields~\citep{ravikumar2010high}, non-parametric regression~\citep{ravikumar2007spam}, diffusion networks~\citep{daneshmand2014estimating}, among others. 

\begin{assumption}[Mutual Incoherence]
	\label{assum:mutual incoherence condition}
	$\max( \| H_{P^c P} H_{PP}^{-1} \|_{\infty},  \| H_{Q^c Q} H_{QQ}^{-1} \|_{\infty}) \leq 1 - \xi$ for some $\xi \in (0, 1]$.
\end{assumption}
Again, we will show that with a sufficient number of samples, a condition similar to Assumption \ref{assum:mutual incoherence condition} holds in the finite-sample setting with high probability.
\begin{lemma}
	\label{lem:sample mutual incoherence condition}
If Assumption \ref{assum:mutual incoherence condition} holds and $n_1 = \Omega(\frac{s_1^3 (\log s_1 + \log d)}{\tau(C_{\min}, \xi, \sigma, \Sigma)})$ and $n_2 = \Omega(\frac{s_2^3 (\log s_2 + \log d)}{\tau(C_{\min}, \xi, \sigma, \Sigma)})$, then 
\begin{align*}
\max( \| \widehat{H}_{P^c P} \widehat{H}_{PP}^{-1} \|_{\infty},  \| \widehat{H}_{Q^c Q} \widehat{H}_{QQ}^{-1} \|_{\infty})  \leq 1 - \frac{\xi}{2}
\end{align*}
with probability at least $1 - \calO(\frac{1}{d})$ where $\tau(C_{\min}, \xi, \sigma, \Sigma)$ is a constant independent of $n_1, n_2, d, s_1$ and $s_2$.
\end{lemma}


\subsection{Main Theorem}
\label{subsec:main theorem}
Now we are ready to state our main result.
\begin{theorem}
	\label{thm:main theorem}
If Assumptions~\ref{assum:identifiability},~\ref{assum:postive definite} and \ref{assum:mutual incoherence condition} hold,  $\lambda_1 \geq \frac{64 \rho \sigma_e}{\xi} \sqrt{n_1\log d}$,  $\lambda_2 \geq \frac{64 \rho \sigma_e}{\xi} \sqrt{n_2\log d}$ and $n_1 = \Omega( \frac{s_1^3 \log^2 d}{\tau_0(C_{\min}, \xi, \sigma, \Sigma, \rho)} )$ and $n_2 = \Omega( \frac{s_2^3 \log^2 d}{\tau_0(C_{\min}, \xi, \sigma, \Sigma, \rho)} )$, then solution to optimization problem \eqref{eq:opt prob invex mlr} satisfies the following properties:
	\begin{enumerate}
		\item The labels are recovered exactly, i.e., 
		\begin{align}
		\begin{split}
				t_i = t_i^*,\; \forall i \in \seq{n}
		\end{split}
		\end{align}
		\item The regression parameters vectors are close to the true vectors. Formally,
		\begin{align}
		\begin{split}
			W = \begin{bmatrix} \beta_1 \\ 1 \end{bmatrix} \begin{bmatrix} \beta_1^\T & 1 \end{bmatrix}, \; U = \begin{bmatrix} \beta_2 \\ 1 \end{bmatrix} \begin{bmatrix} \beta_2^\T & 1 \end{bmatrix} 
		\end{split}
		\end{align}
		such that $\beta_1 = \begin{bmatrix} \beta_{1_P} & \bzero_{P^c} \end{bmatrix}^\T$ and $\beta_2 = \begin{bmatrix} \beta_{2_Q} & \bzero_{Q^c} \end{bmatrix}^\T$ and 
		\begin{align}
		\begin{split}
			\| \beta_1 - \beta_1^* \|_2 &\leq   (2 +  b_1)\frac{2\lambda_1 \sqrt{s_1}}{C_{\min} n_1} \\
			\| \beta_2 - \beta_2^* \|_2   &\leq (2 +  b_2)\frac{2\lambda_2 \sqrt{s_2}}{C_{\min} n_2} 
		\end{split}
		\end{align}
	\end{enumerate}
\end{theorem}

In order to prove Theorem~\ref{thm:main theorem},  we will have to show that the labels are recovered exactly. We will also need to show that $W$ and $U$ are rank-1 matrices with eigenvectors $\begin{bmatrix} \beta_1 & 1 \end{bmatrix}^\T$ and $\begin{bmatrix} \beta_2 & 1 \end{bmatrix}^\T$ respectively. Moreover, we will also need to ensure that their supports match supports of the true vectors and they are close to true vectors in $\ell_2$-norm.

\section{Theoretical Analysis}
\label{sec:theoretical analysis}

We use primal-dual witness approach to show our results. The primal-dual witness approach was developed by~\cite{wainwright2009info} for linear regression problem which has been later used in many convex problems such as  Markov random fields~\citep{ravikumar2010high}, non-parametric regression~\citep{ravikumar2007spam}, diffusion networks~\citep{daneshmand2014estimating} etc. The main idea is to start with a potential solution with certain properties and then later show that these properties are indeed consistent with the final solution. We extend this idea to our invex problem. To that end, we start our proof with a potential solution which has certain ``consistency certificate''.

\subsection{Consistency Certificate}
\label{subsec:consistency certificates}

We start by taking solutions $W$ and $U$ with the following properties which we call consistency certificates:
\begin{enumerate}
	\item[C1.] $W$ and $U$ are sparse. In particular, they have the following sparsity structure:
	\begin{align}
	\begin{split}
		W &= \begin{bmatrix}
		W_{PP} & \bzero_{PP^c} & W_{P d+1} \\
		\bzero_{P^cP} & \bzero_{P^cP^c} & \bzero_{P^c d+1} \\
		W_{d+1 P} & \bzero_{d+1 P^c} & W_{d+1,d+1} 
		\end{bmatrix} \\
		U &= \begin{bmatrix}
		U_{QQ} & \bzero_{QQ^c} & U_{Q d+1} \\
		\bzero_{Q^cQ} & \bzero_{Q^cQ^c} & \bzero_{Q^c d+1} \\
		U_{d+1 Q} & \bzero_{d+1 Q^c} & U_{d+1,d+1} 
		\end{bmatrix} 
	\end{split}
	\end{align}
	We collect all the non-zero entries of $W$ and $U$ in $\overline{W} \in \real^{s_1+1, s_1+1}$ and $\overline{U} \in \real^{s_2+1, s_2+1}$.
\end{enumerate}
It should be noted that consistency certificate is not another assumption. In that, eventually we will have to show that it holds in the final solution. We can prove that C1 is consistent with final solution by showing strict dual feasibility for both $W$ and $U$ which we do in subsection~\ref{subsec: validating consistency certificate}.       

\subsection{A Modified Compact Invex Problem}
\label{subsec:a modified invex problem}

Once we substitute $W$ and $U$ from C1 in optimization problem \eqref{eq:opt prob invex mlr}, we get a low dimensional optimization problem. 
\begin{definition}[Compact Invex MLR]
	\label{def:compact invex mlr}
	\begin{align}
	\label{eq:opt prob invex mlr compact}
	\begin{split}
	\begin{matrix}
	\min\limits_{t, \overline{W}, \overline{U}} & \bar{f}(t, \overline{W}, \overline{U}) + \lambda_1 \bar{g}(t, \overline{W}, \overline{U}) + \lambda_2 \bar{h}(t, \overline{W}, \overline{U}) \\
	\text{such that} & \\
	& \overline{W} \succeq \mathbf{0}, \; \overline{U} \succeq \mathbf{0} \\
	& \overline{W}_{s_1+1, s_1+1} = 1, \;  \overline{U}_{s_2+1, s_2+1} = 1 \\
	& \| t \|_{\infty} \leq 1 
	\end{matrix} \, .
	\end{split}
	\end{align}
	where $ \bar{f}(t, \overline{W}, \overline{U}) = \sum_{i=1}^n \frac{1}{2} (\inner{\overline{S}_i^P}{\overline{W}} + \inner{\overline{S}_i^Q}{\overline{U}}) + \sum_{i=1}^n \frac{1}{2} t_i (\inner{\overline{S}_i^P}{\overline{W}} - \inner{\overline{S}_i^Q}{\overline{U}}), \bar{g}(t, \overline{W}, \overline{U}) = \|\overline{W}(:)\|_1$, $\bar{h}(t, \overline{W}, \overline{U}) = \| \overline{U}(:)\|_1$ and $\lambda_1$ and $\lambda_2$ are positive regularizers.
\end{definition}
Note that 
\begin{align}
\begin{split}
	\overline{S}_i^P = \begin{bmatrix}
	S_{i_{P, P}} & S_{i_{P, d+1}} \\
	S_{i_{d+1, P}} & S_{i_{d+1, d+1}}
	\end{bmatrix}, \; \overline{S}_i^Q = \begin{bmatrix}
	S_{i_{Q, Q}} & S_{i_{Q, d+1}} \\
	S_{i_{d+1, Q}} & S_{i_{d+1, d+1}}
	\end{bmatrix} \; .
\end{split}
\end{align}
For clarity, we  will drop the superscripts from $\overline{S}_i$ when the context is clear. Next, we list down the necessary and sufficient conditions to solve optimization problem \eqref{eq:opt prob invex mlr compact}.

\subsection{Necessary and Sufficient KKT Conditions}
\label{subsec:kkt condtions}

First, we write the Lagrangian $L(\Theta)$ for fixed $\lambda_1 > 0$ and $\lambda_2 > 0$, where $\Theta = (t, \overline{W}, \overline{U}; \Pi, \Lambda, \alpha, \gamma, \nu, \mu)$ is a collection of parameters.
\begin{align}
\label{eq:lagrangian}
\begin{split}
 L(\Theta) =& \bar{f}(t, \overline{W}, \overline{U}) + \lambda_1 \bar{g}(t, \overline{W}, \overline{U}) + \lambda_2 \bar{h}(t, \overline{W}, \overline{U}) - \inner{\Pi}{\overline{W}} - \inner{\Lambda}{\overline{U}} +  \alpha ( \overline{W}_{s_1+1, s_1+1} - 1 ) +\\
  &\gamma ( \overline{U}_{s_2+1, s_2+1} - 1 ) - \sum_{i=1}^n \nu_i (t_i + 1) +  \sum_{i=1}^n \mu_i (t_i - 1)
\end{split}
\end{align} 

Here $\Pi \succeq \bzero, \Lambda \succeq \bzero, \alpha \in \real, \gamma \in \real, \nu_i > 0$ and $\mu_i > 0$ are the dual variables (of appropriate dimensions) for optimization problem \eqref{eq:opt prob invex mlr compact}. Using this Lagrangian, the KKT conditions at the optimum can be written as: 

\begin{enumerate}
	\item Stationarity conditions:
	\begin{align}
	\label{eq:stationarity W}
	\begin{split}
	\sum_{i=1}^n \frac{t_i+1}{2} \overline{S}_i^P + \lambda_1 Z - \Pi + I_{\alpha} = \bzero
	\end{split}
	\end{align}
	where $Z$ is an element of the subgradient set of $\|\overline{W}(:)\|_1$, i.e., $Z \in  \diff{\|\overline{W}(:)\|_1}{\overline{W}}$ and $\| Z(:) \|_{\infty} \leq 1$ and $I_{\alpha} \in \real^{s_1+1, s_1+1}$ has all zero entries except $(s_1+1, s_1+1)$ entry which is $\alpha$.
	\begin{align}
	\label{eq:stationarity U}
	\begin{split}
	\sum_{i=1}^n \frac{1 - t_i}{2} \overline{S}_i^Q + \lambda_2 V - \Lambda + I_{\gamma} = \bzero
	\end{split}
	\end{align}
	where $V$ is an element of the subgradient set of $\|\overline{U}(:)\|_1$, i.e., $V \in  \diff{\|\overline{U}(:)\|_1}{\overline{U}}$ and $\| V(:) \|_{\infty} \leq 1$ and $I_{\gamma} \in \real^{s_2+1, s_2+1}$ has all zero entries except $(s_2+1, s_2+1)$ entry which is $\gamma$. 
	\begin{align}
	\label{eq:stationarity t}
	\begin{split}
	&\frac{1}{2} \inner{\overline{S}_i^P}{\overline{W}} - \frac{1}{2}  \inner{\overline{S}_i^Q}{\overline{U}} - \nu_i + \mu_i= 0, \forall i \in \seq{n} \\
	\end{split}
	\end{align}
	\item Complementary Slackness conditions:
	\begin{align}
	\label{eq:complimentarity Pi Lambda}
	\begin{split}
		\inner{\Pi}{\overline{W}} = 0, \; \inner{\Lambda}{\overline{U}} = 0 
	\end{split}
	\end{align}
	\begin{align}
	\label{eq:complimentarity t}
	\begin{split}
	\nu_i (t_i + 1) = 0,\; \mu_i (t_i - 1) = 0 \; \forall i \in \seq{n}
	\end{split}
	\end{align}
	\item Dual Feasibility conditions:
	\begin{align}
	\label{eq:dual feasibility Pi Lambda}
	\begin{split}
	\Pi \succeq \bzero, \; \Lambda \succeq \bzero  
	\end{split}
	\end{align}
	\begin{align} 
	\label{eq:dual feasibility t}
	\begin{split}
		\nu_i \geq 0, \; \mu_i \geq 0\; \forall i\in \seq{n} 
	\end{split}
	\end{align}
	\item Primal Feasibility conditions:
	\begin{align}
	\label{eq:primal feasibility}
	\begin{split}
	& \overline{W} \succeq \mathbf{0}, \; \overline{U} \succeq \mathbf{0} \\
	& \overline{W}_{s_1+1, s_1+1} = 1, \;  \overline{U}_{s_2+1, s_2+1} = 1 , \quad \| t \|_{\infty} \leq 1 
	\end{split}
	\end{align}
\end{enumerate}

Next, we will provide a setting for primal and dual variables which satisfies all the KKT conditions.

\subsection{Construction of Primal and Dual Variables}
In this subsection, we will provide a construction of primal and dual variables which satisfies the KKT conditions for optimization problem \eqref{eq:opt prob invex mlr compact}. To that end, we provide our first main result. 

\begin{theorem}[Primal Dual Variables Construction]
	\label{thm:primal dual witness construction}
If Assumptions~\ref{assum:identifiability},~\ref{assum:postive definite} and \ref{assum:mutual incoherence condition} hold, $\lambda_1 \geq \frac{64 \rho \sigma_e}{\xi} \sqrt{n_1\log d}$, $\lambda_2 \geq \frac{64 \rho \sigma_e}{\xi} \sqrt{n_2\log d}$ and $n_1 = \Omega( \frac{s_1^3 \log^2 d}{\tau_0(C_{\min}, \xi, \sigma, \Sigma, \rho)} )$ and $n_2 = \Omega( \frac{s_2^3 \log^2 d}{\tau_0(C_{\min}, \xi, \sigma, \Sigma, \rho)} )$, then the following setting of primal and dual variables 

	\begin{align}
	\label{eq:primal dual variable setting}
	\begin{split}
	&\text{Primal Variables:} \\
	&\quad t_i = t_i^*, \; \forall i \in \seq{n}, \;\;\; \overline{W} = \begin{bmatrix} \tilde{\beta}_1 \\ 1 \end{bmatrix} \begin{bmatrix} \tilde{\beta}_1^\T & 1 \end{bmatrix} , \;\;\; \overline{U} = \begin{bmatrix} \tilde{\beta}_2 \\ 1 \end{bmatrix} \begin{bmatrix} \tilde{\beta}_2^\T & 1 \end{bmatrix}  \\
	&\quad \text{where }\\
	&\quad \tilde{\beta}_1 = \arg\min\limits_{\beta \in \real^{s_1}} \sum_{i=1}^n \frac{t_i^* + 1}{2} (y_i - X_{i_P}^\T \beta )^2 + \lambda_1 (\| \beta \|_1 + 1)^2, \;\;\; \tilde{\beta}_2 = \arg\min\limits_{\beta \in \real^{s_2}} \sum_{i=1}^n \frac{1 - t_i^*}{2} (y_i - X_{i_Q}^\T \beta )^2 + \lambda_2 (\| \beta \|_1 + 1)^2 \\ 
	&\text{Dual Variables:}\\
	& \quad \nu_i = 0, \mu_i = - \frac{1}{2} \inner{\overline{S}_i^P}{\overline{W}} + \frac{1}{2}  \inner{\overline{S}_i^Q}{\overline{U}}   \;\forall i \in \calI_1, \;\;\; \mu_i = 0, \nu_i =  \frac{1}{2} \inner{\overline{S}_i^P}{\overline{W}} - \frac{1}{2}  \inner{\overline{S}_i^Q}{\overline{U}}   \;\forall i \in \calI_2 \\
	&\quad \Pi = \sum_{i=1}^n \frac{t_i^*+1}{2} \overline{S}_i^P + \lambda_1 Z + I_{\alpha}, \;\;\; \Lambda = \sum_{i=1}^n \frac{1-t_i^*}{2} \overline{S}_i^Q + \lambda_2 V + I_{\gamma} \\
	&\quad \alpha = - \inner{ \sum_{i=1}^n \frac{t_i+1}{2} \overline{S}_i^P + \lambda_1 Z}{\overline{W}}, \;\;\; \gamma = - \inner{ \sum_{i=1}^n \frac{1 - t_i}{2} \overline{S}_i^Q + \lambda_2 V}{\overline{U}}
	\end{split}
	\end{align}
	satisfies all the KKT conditions for optimization problem \eqref{eq:opt prob invex mlr compact} with probability at least $1 - \calO(\frac{1}{d})$, where $\tau_0(C_{\min}, \alpha, \sigma, \Sigma, \rho, \gamma)$ is a constant independent of $s_1, s_2, d, n_1$ and $n_2$  and thus, the primal variables are a globally optimal solution for \eqref{eq:opt prob invex mlr compact}. Furthermore, the above solution is also unique.
\end{theorem}
\paragraph{Proof Sketch.} The main idea behind our proofs is to verify that the setting of primal and dual variables in Theorem~\ref{thm:primal dual witness construction} satisfies all the KKT conditions described in subsection \ref{subsec:kkt condtions}. We do this by proving multiple lemmas in subsequent subsections. The outline of the proof is as follows:
\begin{itemize}
	\item It can be trivially verified that the primal feasibility condition \eqref{eq:primal feasibility} holds. The stationarity conditions \eqref{eq:stationarity W} and \eqref{eq:stationarity U} holds by construction of $\Pi$ and $\Lambda$ respectively. Similarly, the \eqref{eq:stationarity t} holds by choice of $\nu_i$ and $\mu_i$. Choice of $t, \nu_i, \mu_i, \alpha$ and $\gamma$ ensure that complementary slackness conditions~\eqref{eq:complimentarity Pi Lambda} and ~\eqref{eq:complimentarity t} also hold. 
	\item In subsection~\ref{subsec:verifying dual feasibility}, we use  Lemmas~\ref{lem:bound on Delta},~\ref{lem:zero eigenvalue} and \ref{lem:strictly positive second eigenvalue} to verify that the dual feasibility conditions~\eqref{eq:dual feasibility t} and \eqref{eq:dual feasibility Pi Lambda} hold. We will also show in subsection subsection~\ref{subsec:verifying dual feasibility} that our solution is also unique.
\end{itemize}  

\subsection{Verifying Dual Feasibility}
\label{subsec:verifying dual feasibility}

To verify dual feasibility, first we will show that $\mu_i \geq 0, \forall i \in \calI_1$, $\nu_i \geq 0, \forall i \in \calI_2$. We define $\Delta_1 \triangleq \tilde{\beta}_1 - \beta_{1_P}^*$ and $\Delta_2 \triangleq \tilde{\beta}_2 - \beta_{2_Q}^*$. Then, the following lemma holds true.

\begin{lemma}
	\label{lem:bound on Delta}
If Assumptions~\ref{assum:identifiability},~\ref{assum:postive definite} and \ref{assum:mutual incoherence condition} hold, and $\lambda_1 \geq 8 \rho \sigma_e \sqrt{n_1\log d}$, $\lambda_2 \geq 8 \rho \sigma_e \sqrt{n_2\log d}$, $n_1 = \Omega( \frac{s_1^3 \log d}{\tau(C_{\min}, \xi, \sigma, \Sigma)})$, and $n_2 = \Omega( \frac{s_2^3 \log d}{\tau(C_{\min}, \xi, \sigma, \Sigma)})$ then 
$\| \Delta_1 \|_2 \leq (2 +  b_1)\frac{2\lambda_1 \sqrt{s_1}}{C_{\min} n_1} $ and $\| \Delta_2 \|_2 \leq (2 +  b_2)\frac{2\lambda_2 \sqrt{s_2}}{C_{\min} n_2} $ with probability at least $1 - \calO(\frac{1}{d})$ where $\tau(C_{\min}, \xi, \sigma, \Sigma)$ is a constant independent of $s_1, s_2, d, n_1$ or $n_2$.
\end{lemma}

Using the result of Lemma~\ref{lem:bound on Delta}, we are going to prove that the settings for dual variables $\mu_i$ and $\nu_i$ works with high probability.
\begin{lemma}
	\label{lem:setting of mu_i and nu_i}
	If Assumptions~\ref{assum:identifiability},~\ref{assum:postive definite} and \ref{assum:mutual incoherence condition} hold, and $\lambda_1 \geq 8 \rho \sigma_e \sqrt{n_1\log d}$, $\lambda_2 \geq 8 \rho \sigma_e \sqrt{n_2\log d}$, $n_1 = \Omega( \frac{s_1^3 \log d}{\tau(C_{\min}, \xi, \sigma, \Sigma)})$, and $n_2 = \Omega( \frac{s_2^3 \log d}{\tau(C_{\min}, \xi, \sigma, \Sigma)})$ then $\mu_i \geq 0, \forall i \in \calI_1$ and $\nu_i \geq 0, \forall i \in \calI_2$ with probability at least $1 - \calO(\frac{1}{d})$ where $\tau(C_{\min}, \xi, \sigma, \Sigma)$ is a constant independent of $s_1, s_2, d, n_1$ or $n_2$.
\end{lemma} 

Now we will show that $\Pi \succeq \bzero$ and $\Lambda \succeq \bzero$. We will do this in two steps. The first step is to show that both $\Pi$ and $\Lambda$ have a zero eigenvalue. In particular,
\begin{lemma}
	\label{lem:zero eigenvalue} 
	Both $\Pi$ and $\Lambda$ have zero eigenvalues corresponding to eigenvectors $\begin{bmatrix} \tilde{\beta}_1 \\ 1 \end{bmatrix}$ and $\begin{bmatrix} \tilde{\beta}_2 \\ 1 \end{bmatrix}$ respectively.
\end{lemma}
Next, we show that all the other eigenvalues of both $\Pi$ and $\Lambda$ are strictly positive. 
\begin{lemma}
	\label{lem:strictly positive second eigenvalue} 
If Assumption \ref{assum:postive definite} holds and $n_1 = \Omega(\frac{s_1 + \log d}{C_{\min}^2})$ and $n_2 = \Omega(\frac{s_2 + \log d}{C_{\min}^2})$, then the second eigenvalues of  $\Pi$ and $\Lambda$ are strictly positive with probability at least $1 - \calO(\frac{1}{d})$, i.e.,  $\eig_2(\Pi) > 0$ and $\eig_2(\Lambda) > 0$.
\end{lemma}

On the one hand, Lemma~\ref{lem:strictly positive second eigenvalue} ensures that $\Pi \succeq 0$ and $\Lambda \succeq 0$, but on the other it also forces $\overline{W}$ and $\overline{U}$ to be rank-1 and unique as both $\Pi$ and $\overline{W}$ have to be positive semidefinite and $\Pi$ has exactly one vector in its nullspace (same with $\Lambda$ and $\overline{U}$). 

\subsection{Going back to Invex MLR}
\label{subsec:going back to invex MLR}

Now that we have the setting of $t_i, \overline{W}$ and $\overline{U}$ for Compact Invex MLR problem \eqref{eq:opt prob invex mlr compact}, we can extend these to the original Invex MLR problem \eqref{eq:opt prob invex mlr}. Notice that all the other entries of $W$ and $U$ are zeros, thus it readily follows that
\begin{align}
\begin{split}
W = \begin{bmatrix} \beta_1 \\ 1 \end{bmatrix} \begin{bmatrix} \beta_1^\T & 1 \end{bmatrix}, \; U = \begin{bmatrix} \beta_2 \\ 1 \end{bmatrix} \begin{bmatrix} \beta_2^\T & 1 \end{bmatrix} 
\end{split}
\end{align}
where $\beta_1 = \begin{bmatrix} \tilde{\beta}_1 \\ \bzero \end{bmatrix}$ and $\beta_2 = \begin{bmatrix} \tilde{\beta}_2 \\ \bzero \end{bmatrix}$. Furthermore, result from Lemma~\ref{lem:bound on Delta} extends directly and gives us
\begin{align}
\begin{split}
			\| \beta_1 - \beta_1^* \|_2 \leq   (2 +  b_1)\frac{2\lambda_1 \sqrt{s_1}}{C_{\min} n_1} ,\;\;\;	\| \beta_2 - \beta_2^* \|_2   \leq (2 +  b_2)\frac{2\lambda_2 \sqrt{s_2}}{C_{\min} n_2}  \; .
\end{split}
\end{align}

The last remaining thing is to show that consistency certificate C1 indeed holds which we will do in next subsection.

\subsection{Validating Consistency Certificate}
\label{subsec: validating consistency certificate}

\begin{figure*}[!ht]
	\centering
	\begin{subfigure}{.33\textwidth}
		\centering
		\includegraphics[width=\linewidth]{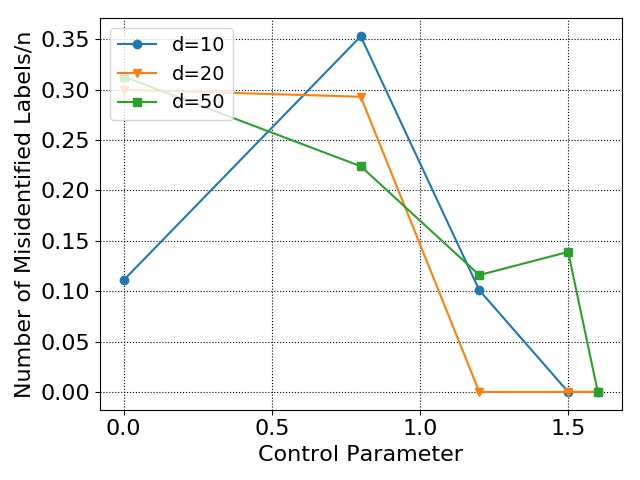}
		\caption{Misidentified labels (in ratio to $n$)}
		\label{fig:recnumsample}
	\end{subfigure}%
	\begin{subfigure}{.33\textwidth}
		\centering
		\includegraphics[width=\linewidth]{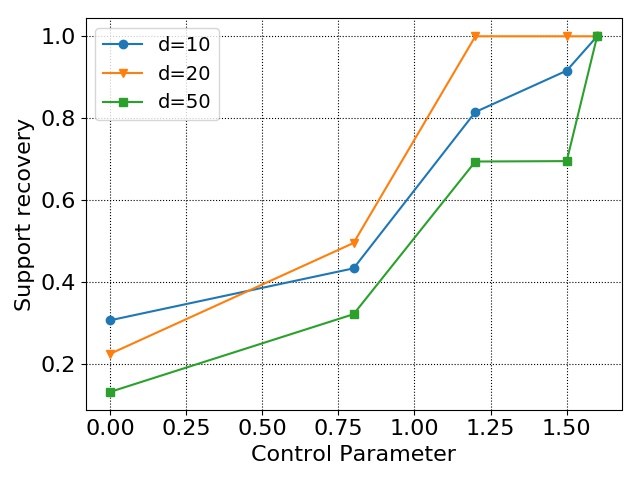}	
		\caption{Support recovery of $\beta_1^*$}
		\label{fig:recnumsamplecp}
	\end{subfigure}%
	\begin{subfigure}{.33\textwidth}
		\centering
		\includegraphics[width=\linewidth]{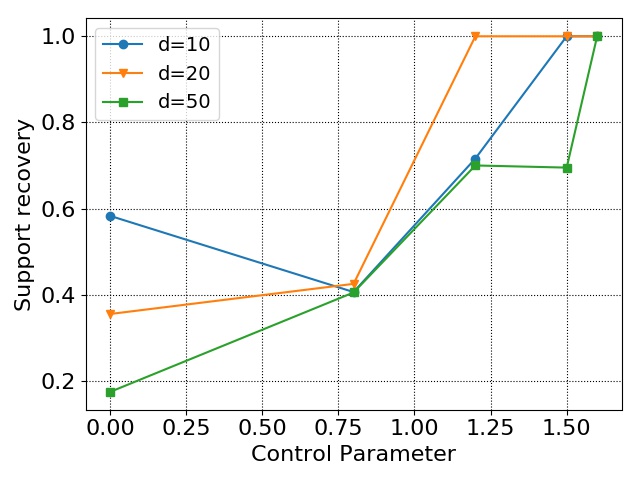}
		\caption{Support recovery of $\beta_2^*$}
		\label{fig:labelnumsample}
	\end{subfigure}
	\caption{Label and support recovery with control parameter $C_p$.}
	\label{fig:recovery}
\end{figure*}

Observe that once we substitute $t_i = t_i^*$ in optimization problem \eqref{eq:opt prob invex mlr}, it decouples in to two independent convex optimization problems involving $W$ and $U$ respectively. Furthermore, since we established that $W$ and $U$ are rank-1, we can rewrite these independent problems in terms of $\beta_1$ and $\beta_2$. Our task is to show that $\beta_{1_{P^c}} = \bzero$ and  $\beta_{2_{Q^c}} = \bzero$. It suffices to show it for $\beta_1$ as arguments for $\beta_2$ are the same. Below, we consider the simplified optimization problem in terms of $\beta_1$:
\begin{align}
\begin{split}
	\beta_1 = \arg\min_{\beta \in \real^d} \sum_{i \in \calI_1}(X_i^\T \beta - y_i )^2 + \lambda_1 (\| \beta \|_1 + 1)^2
\end{split}
\end{align}

Since we are only dealing with measurements in $\calI_1$, we can substitute $y_i = X_i^\T \beta^* + e_i$. Furthermore, $\beta_1$ must satisfy stationarity KKT condition which can be written as:
\begin{align}
\label{eq:grad}
\begin{split}
	&\frac{1}{n_1}\sum_{i \in \calI_1} X_i X_i^\T (\beta_1 - \beta^*) - \frac{1}{n_1}\sum_{i \in \calI_1} X_i e_i + \frac{1}{n_1}\lambda_1 ( \| \beta_1 \|_1 + 1) z = \bzero, 
\end{split}
\end{align} 
where $z$ is in subdifferential of $\| \beta_1 \|_1$ and $\| z \|_\infty \leq 1$. Specifically, $z_i = \sign(\beta_1), \forall i \in P$ and $z_i \in [-1, 1], \forall i \in P^c$. Our task is to show that $z$ follows strict dual feasibility, i.e., $\| z_{P^c} \|_\infty < 1$. We decompose equation~\eqref{eq:grad} in to two parts -- one corresponding to entries in $P$ and the other corresponding to entries in $P^c$. For entries in $P$, we have
\begin{align}
\label{eq:grad support}
\begin{split}
	&\frac{1}{n_1}\sum_{i \in \calI_1} X_{i_P} X_{i_P}^\T  (\beta_{1_P} - \beta_{1_P}^*) - \frac{1}{n_1}\sum_{i \in \calI_1} X_{i_P} e_i
	+\frac{1}{n_1} \lambda_1( \| \beta_1 \|_1 + 1) z_P = \bzero
\end{split}
\end{align}
Similarly, for entries in $P^c$, we have
\begin{align*}
\begin{split}
&\frac{1}{n_1}\sum_{i \in \calI_1} X_{i_{P^c}} X_{i_P}^\T (\beta_{1_P} - \beta_{1_P}^*)   - \frac{1}{n_1}\sum_{i \in \calI_1} X_{i_{P^c}} e_i +\frac{1}{n_1} \lambda_1 ( \| \beta_1 \|_1 + 1)  z_{P^c} = \bzero
\end{split}
\end{align*}
After rearranging the terms and substituting for $(\beta_{1_P} - \beta_{1_P}^*) $ from equation~\eqref{eq:grad support}, we get
\begin{align*}
\begin{split}
	&\frac{\lambda_1}{n_1}(1 + \| \beta_1 \|_1) z_{P^c} = -  \widehat{H}_{P^cP}  \widehat{H}_{PP}^{-1} ( \frac{1}{n_1}\sum_{i \in \calI_1} X_{i_P} e_i - \frac{1}{n_1} \lambda_1( \| \beta_1 \|_1 + 1) z_P ) + \frac{1}{n_1}\sum_{i \in \calI_1} X_{i_{P^c}} e_i
\end{split}
\end{align*}

Let $\bar{\lambda}_1 = \frac{\lambda_1}{n_1}$ and note that $\| \beta_1 \|_1 \geq 0$, using norm inequalities we can rewrite the above equation as:
\begin{align*}
\begin{split}
	&\| z_{P^c} \|_\infty \leq \| \widehat{H}_{P^cP}  \widehat{H}_{PP}^{-1}  \|_{\infty} ( \| \frac{1}{\bar{\lambda}_1}  \frac{1}{n_1}\sum_{i \in \calI_1} X_{i_P} e_i \|_\infty + \| z_P \|_\infty ) + \| \frac{1}{\bar{\lambda}_1} \frac{1}{n_1}\sum_{i \in \calI_1} X_{i_{P^c}} e_i \|_\infty
\end{split}
\end{align*}

We know that $\| \widehat{H}_{P^cP}  \widehat{H}_{PP}^{-1}  \|_{\infty} \leq (1 - \frac{\xi}{2})$ for some $\xi \in (0, 1]$. The following lemma provides bounds on  $\| \frac{1}{\bar{\lambda}_1}  \frac{1}{n_1}\sum_{i \in \calI_1} X_{i_P} e_i \|_\infty$ and $\| \frac{1}{n_1}\sum_{i \in \calI_1} X_{i_{P^c}} e_i \|_\infty$.

\begin{lemma}
	\label{lem:bound X_se and X_s^ce}
Let $\lambda_1 \geq \frac{64 \rho \sigma_e}{\xi} \sqrt{n_1\log d}$. Then the following holds true:
\begin{align*}
\begin{split}
&\prob(\| \frac{1}{\bar{\lambda}_1 } \frac{1}{n_1} \sum_{i \in \calI_1} X_{i_P} e_i \|_{\infty} \geq \frac{\xi}{8 - 4\xi}) \leq \calO(\frac{1}{d}), \;\;\; \prob(\| \frac{1}{\bar{\lambda}_1} \frac{1}{n_1} \sum_{i \in \calI_1} X_{i_{P^c}} e_i \|_{\infty} \geq \frac{\xi}{8}) \leq \calO(\frac{1}{d}) 
\end{split}
\end{align*}
\end{lemma}

It follows that $\| z_{P^c} \|_\infty \leq 1 - \frac{\xi}{4}$ with probability at least $1 - \calO(\frac{1}{d})$. Thus, C1 indeed holds with high probability.

\section{Experimental Validation}
\label{sec:experimental validation}

Note that we are not proposing any new algorithm in our paper. However, to validate our theoretical results we performed experiments on synthetic data. We generated response $y$ using Gaussian random variables $X$ and chose regression parameter $\beta^*_1$ (or $\beta^*_2$) based on the label of the samples. We fixed the sparsity $s_1= s_2 = 4$, however supports were not necessarily the same for both the regression parameter vectors.  We varied $n_1$ and $n_2$ according to our theorems, i.e., both were varied with $ 10^{C_p} \log^2 d$ for $d=10, 20$ and $50$ where $C_p$ is a control parameters. The regularizers were kept according to our theorem and were varied as $\calO(\sqrt{n_1 \log d})$ and  $\calO(\sqrt{n_2 \log d})$. We measured performance of our algorithm based on the label recovery (in ratio to supplied $n$) and support recovery for both parameter vectors. The experiments were run three times independently. Note how we make zero mistakes as we increase number of samples. Similarly, support recovery (ratio of intersection and union with correct support) for both parameter vectors goes to $1$ as we increase sample size. It should be noted that while we do not propose any algorithm but our method is free of any initialization requirement. A projected subgradient method is used to check convergence for our problem which is achieved without any requirement on initialization. In fact, any algorithm which converges to a stationary point should work for our framework.

\section{Concluding Remarks}
\label{sec:conclusion}

We provide a novel formulation of invex MLR. We show that invexity of our optimization problem allows for a tractable solution. We provide provable theoretical guarantees for our solution. The sample complexity of our method is polynomial in terms of sparsity and logarithmic in terms of the dimension of the true parameter. Our method helps to identify labels exactly and recovers regression parameter vectors with correct support and correct sign. It would be interesting to think about extending our ideas to mixture of more than two groups of regressions in future.


\appendix
\onecolumn


\section{Continuous Relaxation of Standard MLR is non-convex}
\label{sec:std mlr nonconvex}

It suffices to prove that the objective function $l(z, \beta_1, \beta_2)$ of optimization problem~\eqref{eq:opt prob standard mlr} is non-convex when $z_i$ is allowed to be between $0$ and $1$. We note that
\begin{align}
\label{eq:mix_lin}
\begin{split}
\begin{matrix}
	l(z, \beta_1, \beta_2) =  \sum_{i=1}^n z_i (y_i - X_i^\T \beta_1)^2 + (1 - z_i) (y_i - X_i^\T \beta_2)^2
\end{matrix}
\end{split}
\end{align}
where $\beta_1 \in \real^d, \beta_2 \in \real^d$ and $z_i \in [0, 1], \forall i \in \seq{n}$. Let $\Theta = (z, \beta_1, \beta_2, \bar{z}, \bar{\beta}_1, \bar{\beta}_2)$. We consider the following quantity:
\begin{align}
\begin{split}
F(\Theta) = f(z, \beta_1, \beta_2) - f(\bar{z}, \bar{\beta}_1, \bar{\beta}_2) - \sum_{i=1}^n \diff{f}{\bar{z}_i} (z_i - \bar{z}_i) - \diff{f}{\bar{\beta}_1}^\T (\beta_1 - \bar{\beta}_1) -  \diff{f}{\bar{\beta}_2}^\T (\beta_2 - \bar{\beta}_2)
\end{split}
\end{align}

where
\begin{align}
\begin{split}
\diff{f}{\bar{z}_i} &= (y_i - X_i^\T \bar{\beta}_1)^2 - (y_i - X_i^\T \bar{\beta}_2)^2, \quad \forall i \in \seq{n} \\
\diff{f}{\bar{\beta}_1} &= \sum_{i=1}^n -2 \bar{z}_i X_i (y_i - X_i^\T \bar{\beta}_1) \\ 
\diff{f}{\bar{\beta}_2} &= \sum_{i=1}^n -2(1 - \bar{z}_i) X_i (y_i - X_i^\T \bar{\beta}_2)
\end{split}
\end{align}

It suffices to show that $F(\Theta)$ changes sign for different feasible values of $\Theta$. We choose the following variables:
\begin{align}
\begin{split}
	z_i = 0, \bar{z}_i = \frac{1}{2} \quad \forall i \in \seq{n} \\
	\beta_{1_k} = u_1, \beta_{1_j} = 0, \forall j \ne k \\  
	\beta_{2_l} = u_2, \beta_{2_j} = 0, \forall j \ne l \\  
	\bar{\beta}_{1_k} = w_1, \bar{\beta}_{1_j} = 0, \forall j \ne k \\  
	\bar{\beta}_{2_l} = w_2, \bar{\beta}_{2_j} = 0, \forall j \ne l \\
	u_1 = w_1 - (u_2 - w_2)
\end{split}
\end{align}
Note that choice of $w_1, u_2$ and $w_2$ can be arbitrary. This simplifies $F(\Theta)$:
\begin{align}
\begin{split}
	F(\Theta) = \sum_{i=1}^n (y_i - u_2 X_{il})^2 -  \sum_{i=1}^n (y_i - w_2 X_{il})^2
\end{split}
\end{align}

Consider the case when $X_{il} > 0, \forall i \in \seq{n}$. Then choosing $u_2 < w_2$ makes $F(\Theta) > 0$ while choosing $u_2 > w_2$ makes $F(\Theta) < 0$. This proves our claim.

\section{Proof of Lemma~\ref{lem:invexity}}
\label{proof:lem invexity}

\paragraph{Lemma~\ref{lem:invexity}}
\emph{For $(t, W, U) \in C$, the functions $ f(t, W, U) = \sum_{i=1}^n \frac{1}{2} \inner{S_i}{W + U} + \sum_{i=1}^n \frac{1}{2} t_i \inner{S_i}{W - U}, g(t, W, U) = \|W(:)\|_1$ and $h(t, W, U) = \| U(:)\|_1$ are $\eta$-invex for $\eta(t, \bar{t}, W, \overline{W}, U, \overline{U}) \triangleq \begin{bmatrix} \eta_t \\ \eta_{W} \\ \eta_{U}\end{bmatrix}$, where $\eta_t = \mathbf{0} \in \real^n, \eta_{W} = -\overline{W}$ and $\eta_{U} = - \overline{U}$. We abuse the vector/matrix notation (by ignoring the dimensions) for clarity of presentation, and avoid the vectorization of matrices.
}
\begin{proof}
	We know $f(t, W, U) = \sum_{i=1}^n \frac{t_i + 1}{2} \inner{S_i}{W} + \frac{1 - t_i}{2} \inner{S_i}{U}$. Then,
	\begin{align}
	\begin{split}
	\diff{f}{t_i} &= \frac{1}{2}\inner{S_i}{W - U} \\
	\diff{f}{W} &= \sum_{i=1}^n \frac{t_i + 1}{2}  S_i \\
	\diff{f}{U} &= \sum_{i=1}^n \frac{1 - t_i}{2} S_i 
	\end{split}
	\end{align}
	
	To prove that $f(t, W, U)$ is invex, we need to show that 
	\begin{align}
	\label{eq:invexity_cond}
	\begin{split}
	f(t, W, U) - f(\bar{t}, \bar{W}, \bar{U}) - \sum_{i=1}^n \eta_{t_i} \diff{f}{\bar{z}_i} - \inner{\eta_W}{\diff{f}{\bar{W}}} - \inner{\eta_U}{\diff{f}{\bar{U}}} \geq 0
	\end{split}
	\end{align} 
	We take $\eta_t = \mathbf{0} \in \real^n, \eta_{W} = -\overline{W}$ and $\eta_{U} = - \overline{U}$ and expand LHS of equation~\eqref{eq:invexity_cond} as follows:
	\begin{align}
	\label{eq:invexity_cond1}
	\begin{split}
	&\sum_{i=1}^n \frac{t_i + 1}{2} \inner{S_i}{W} + \frac{1 - t_i}{2} \inner{S_i}{U} - \sum_{i=1}^n \frac{\bar{t}_i + 1}{2} \inner{S_i}{\bar{W}} - \frac{1 - \bar{t}_i}{2} \inner{S_i}{\bar{U}}  + \inner{\overline{W}}{\sum_{i=1}^n \frac{\bar{t}_i + 1}{2} S_i } + \inner{\overline{U}}{\sum_{i=1}^n  \frac{1 - \bar{t}_i}{2} S_i} \\
	&=  \sum_{i=1}^n \frac{t_i + 1}{2} \inner{S_i}{W} + \frac{1 - t_i}{2} \inner{S_i}{U} \\
	& \geq 0 
	\end{split}
	\end{align}
	The last inequality holds because $S_i$, $W$ and $U$ are all positive semidefinite and $t_i \in [-1, 1]$.
	
	Similarly, 
	\begin{align}
	\begin{split}
	&g(t, W, U) - g(\bar{t}, \bar{W}, \bar{U}) - \sum_{i=1}^n \eta_{t_i} \diff{g}{\bar{z}_i} - \inner{\eta_W}{\diff{g}{\bar{W}}} - \inner{\eta_U}{\diff{g}{\bar{U}}} \\
	&= \| W(:) \|_1 - \| \bar{W}(:) \|_1 +  \| \bar{W}(:) \|_1 \geq 0
	\end{split}
	\end{align}
	and 
	\begin{align}
	\begin{split}
	&h(t, W, U) - h(\bar{t}, \bar{W}, \bar{U}) - \sum_{i=1}^n \eta_{t_i} \diff{h}{\bar{z}_i} - \inner{\eta_W}{\diff{h}{\bar{W}}} - \inner{\eta_U}{\diff{h}{\bar{U}}} \\
	&= \| U(:) \|_1 - \| \bar{U}(:) \|_1 +  \| \bar{U}(:) \|_1 \geq 0
	\end{split}
	\end{align}
\end{proof}

\section{Proof of Lemma~\ref{lem:sample positive definite}} 
\label{proof:lem:sample positive definite}

\paragraph{Lemma~\ref{lem:sample positive definite}}
\emph{If Assumption \ref{assum:postive definite} holds and $n_1 = \Omega(\frac{s_1 + \log d}{C_{\min}^2})$ and $n_2 = \Omega(\frac{s_2 + \log d}{C_{\min}^2})$ , then 
	\begin{align*}
	\min( \eig_{\min}(\widehat{H}_{1_{PP}}), \eig_{\min}(\widehat{H}_{2_{QQ}}) ) \geq \frac{C_{\min}}{2}
	\end{align*}
	 and 
	 \begin{align*}
	 \max( \eig_{\max}(\widehat{H}_{1_{PP}}), \eig_{\max}(\widehat{H}_{2_{QQ}}) ) \leq \frac{3C_{\max}}{2}
	 \end{align*}
	  with probability at least $1 - \calO(\frac{1}{d})$.
}
\begin{proof}
	We prove the Lemma for a general support $S$ and samples $n$. The results follow when we substitute $S$ by $P$ and $Q$ and $n$ by $n_1$ or $n_2$ based on the context. By the Courant-Fischer variational representation~\citep{horn2012matrix}:
	\begin{align}
	\begin{split}
	\eig_{\min}(\E(X_iX_i^\T)_{SS}) = \min_{\|y \|_2 = 1} y^\T \E(X_i X_i^\T)_{SS} y &= \min_{\|y \|_2 = 1} y^\T (\E(X_i X_i^\T)_{SS} - \frac{1}{n} X^\T_S X_S + \frac{1}{n} X^\T_S X_S ) y \\
	&\leq y^\T (\E(X_i X_i^\T)_{SS} - \frac{1}{n} X^\T_S X_S + \frac{1}{n} X^\T_S X_S ) y \\
	&= y^\T (\E(X_i X_i^\T)_{SS} - \frac{1}{n} X^\T_S X_S) y + y^\T \frac{1}{n} X^\T_S X_S  y
	\end{split}
	\end{align}
	It follows that
	\begin{align}
	\begin{split}
	\eig_{\min}(\frac{1}{n} X^\T_S X_S) \geq C_{\min} - \| \E(X_i X_i^\T)_{SS} - \frac{1}{n} X^\T_S X_S \|_2
	\end{split}
	\end{align}
	The term $\| \E(X_i X_i^\T)_{SS} - \frac{1}{n} X^\T_S X_S \|_2$ can be bounded using Proposition 2.1 in \cite{vershynin2012close} for sub-Gaussian random variables. In particular,
	\begin{align}
	\begin{split}
	\prob(\| \E(X_i X_i^\T)_{SS} - \frac{1}{n} X^\T_S X_S \|_2 \geq \epsilon) \leq 2  \exp(-c\epsilon^2n + s) 
	\end{split}
	\end{align}
	for some constant $c > 0$. Taking $\epsilon = \frac{C_{\min}}{2}$, we show that $\eig_{\min}(\frac{1}{n}X^\T_S X_S) \geq \frac{C_{\min}}{2}$ with probability at least $1 - 2  \exp(- \frac{cC_{\min}^2 n}{4} + |S|)$. The specific results for $n_1$ and $n_2$ follow directly.
	
	Remark: Similarly, it can be shown that $\eig_{\max}(\frac{1}{n}X^\T_S X_S) \leq \frac{3C_{\max}}{2}$ with probability at least $1 - 2  \exp(- \frac{cC_{\max}^2 n}{4} + |S|)$.
\end{proof}

\section{Proof of Lemma~\ref{lem:sample mutual incoherence condition}}
\label{proof: lem:sample mutual incoherence condition}

\paragraph{Lemma \ref{lem:sample mutual incoherence condition}}
\emph{If Assumption \ref{assum:mutual incoherence condition} holds and $n_1 = \Omega(\frac{s_1^3 (\log s_1 + \log d)}{\tau(C_{\min}, \xi, \sigma, \Sigma)})$ and $n_2 = \Omega(\frac{s_2^3 (\log s_2 + \log d)}{\tau(C_{\min}, \xi, \sigma, \Sigma)})$, then 
	\begin{align*}
	\max( \| \widehat{H}_{P^c P} \widehat{H}_{PP}^{-1} \|_{\infty},  \| \widehat{H}_{Q^c Q} \widehat{H}_{QQ}^{-1} \|_{\infty})  \leq 1 - \frac{\xi}{2}
	\end{align*}
	 with probability at least $1 - \calO(\frac{1}{d})$ where $\tau(C_{\min}, \xi, \sigma, \Sigma)$ is a constant independent of $n_1, n_2, d, s_1$ and $s_2$.
}
\begin{proof}
		We prove the Lemma for a general support $S$ (and corresponding non-support $S^c$) and samples $n$. The results follow when we substitute $S$ by $P$ and $Q$ and $n$ by $n_1$ or $n_2$ based on the context. Let $|S| = s$ and $|S^c| = d - s$. Before we prove the result of Lemma \ref{lem:sample mutual incoherence condition}, we will prove a helper lemma. 
	\begin{lemma}
		\label{lem:helper mutual incoherence}
		If Assumption \ref{assum:mutual incoherence condition} holds then for some $\delta > 0$, the following inequalities hold:
		\begin{align}
		\label{eq:helper mutual incoherence}
		\begin{split}
		&\prob( \| \bHhat_{S^cS} - \bH_{S^cS} \|_{\infty} \geq \delta ) \leq 4 (d - s) s \exp( - \frac{n \delta^2}{128 s^2 (1+4\sigma^2) \max_{l} \Sigma_{ll}^2}) \\
		&\prob( \| \bHhat_{SS} - \bH_{SS} \|_{\infty} \geq \delta ) \leq 4 s^2 \exp( - \frac{n \delta^2}{128 s^2 (1+4\sigma^2) \max_{l} \Sigma_{ll}^2}) \\
		&\prob( \| (\bHhat_{SS})^{-1} - (\bH_{SS})^{-1} \|_{\infty} \geq \delta ) \leq 2\exp(- \frac{c\delta^2 C_{\min}^4n}{4 s} + s)+ 2 \exp( - \frac{cC_{\min}^2n}{4} + s)
		\end{split}
		\end{align}
	\end{lemma}
	\begin{proof}
		\label{proof:helper mutual incoherence}
		Let $A_{ij}$ be $(i, j)$-th entry of $\bHhat_{S^cS} - \bH_{S^cS}$. Clearly, $\E(A_{ij}) = 0$. By using the definition of the $\| \cdot\|_{\infty}$ norm, we can write:
		\begin{align}
		\begin{split}
		\prob(\| \bHhat_{S^cS} - \bH_{S^cS} \|_{\infty} \geq \delta) &= \prob(\max_{i \in S^c} \sum_{j \in S} |A_{ij}| \geq \delta) \\
		& \leq (d - s) \prob(\sum_{j \in S} |A_{ij}| \geq \delta) \\
		&\leq (d - s) s \prob(|A_{ij}| \geq \frac{\delta}{s})
		\end{split}
		\end{align}
		where the second last inequality comes as a result of the union bound across entries in $S^c$ and the last inequality is due to the union bound across entries in $S$. Recall that $X_i, i \in \seq{d}$ are zero mean random variables with covariance $\Sigma$ and each $\frac{X_i}{\sqrt{\Sigma_{ii}}}$ is a sub-Gaussian random variable with parameter $\sigma$. Using the results from Lemma 1 of \cite{ravikumar2011high}, for some $\delta \in (0, s \max_{l} \Sigma_{ll} 8(1 + 4 \sigma^2))$, we can write:
		\begin{align}
		\begin{split}
		\prob(|A_{ij}| \geq \frac{\delta}{s}) \leq 4 \exp( - \frac{n \delta^2}{128 s^2 (1+4\sigma^2) \max_{l} \Sigma_{ll}^2})
		\end{split}
		\end{align} 
		Therefore,
		\begin{align}
		\begin{split}
		&\prob(\| \bHhat_{S^cS} - \bH_{S^cS} \|_{\infty} \geq \delta)  \leq 4 (d - s) s \exp( - \frac{n \delta^2}{128 s^2 (1+4\sigma^2) \max_{l} \Sigma_{ll}^2})
		\end{split}
		\end{align}
		Similarly, we can show that 
		\begin{align}
		\begin{split}
		&\prob(\| \bHhat_{SS} - \bH_{SS} \|_{\infty} \geq \delta)  \leq 4 s^2 \exp( - \frac{n \delta^2}{128 s^2 (1+4\sigma^2) \max_{l} \Sigma_{ll}^2})
		\end{split}
		\end{align}
		Next, we will show that the third inequality in \eqref{eq:helper mutual incoherence} holds. Note that
		\begin{align}
		\begin{split}
		\| (\bHhat_{S^cS})^{-1} - (\bH_{S^cS})^{-1} \|_{\infty} &= \| (\bH_{SS})^{-1} (\bH_{SS} - \bHhat_{SS}) (\bHhat_{SS})^{-1} \|_{\infty} \\
		&\leq \sqrt{s}  \| (\bH_{SS})^{-1} (\bH_{SS} - \bHhat_{SS}) (\bHhat_{SS})^{-1} \|_{2} \\
		&\leq \sqrt{s}  \| (\bH_{SS})^{-1} \|_2 \| (\bH_{SS} - \bHhat_{SS})\|_2 \| (\bHhat_{SS})^{-1} \|_{2} \\
		\end{split}
		\end{align}   
		Note that $\| \bH_{SS} \|_2 \geq C_{\min}$, thus $\| (\bH_{SS})^{-1} \|_2 \leq \frac{1}{C_{\min}}$. Similarly,  $\| \bH_{SS} \|_2 \geq \frac{C_{\min}}{2}$ with probability at least $1 - 2 \exp( - \frac{cC_{\min}^2n}{4} + s)$. We also have $\| (\bH_{SS} - \bHhat_{SS})\|_2 \leq \epsilon$ with probability at least $1 - 2\exp(-c\epsilon^2 n + s)$. Taking $\epsilon = \delta \frac{C_{\min}^2}{2 \sqrt{s}}$, we get 
		\begin{align}
		\begin{split}
		\prob( \| (\bH_{SS} - \bHhat_{SS})\|_2 \geq  \delta \frac{C_{\min}^2}{2 \sqrt{s}} ) \leq 2\exp(- \frac{c\delta^2 C_{\min}^4n}{4 s} + s)
		\end{split}
		\end{align}
		It follows that $\| (\bHhat_{SS})^{-1} - (\bH_{SS})^{-1} \|_{\infty} \leq \delta$ with probability at least $1 - 2\exp(- \frac{c\delta^2 C_{\min}^4n}{4 s} + s) - 2 \exp( - \frac{cC_{\min}^2n}{4} + s)$.
	\end{proof}
	Now we are ready to show that the statement of Lemma \ref{lem:sample mutual incoherence condition} holds using the results from Lemma \ref{lem:helper mutual incoherence}. We will rewrite $\bHhat_{S^cS} (\bHhat_{SS})^{-1}$ as the sum of four different terms:
	\begin{align}
	\begin{split}
	\bHhat_{S^cS} (\bHhat_{SS})^{-1} = T_1 + T_2 + T_3 + T_4,
	\end{split}
	\end{align} 
	where 
	\begin{align}
	\begin{split}
	T_1 &\triangleq \bHhat_{S^cS} ( (\bHhat_{SS})^{-1} - (\bH_{SS})^{-1} ) \\
	T_2 &\triangleq (\bHhat_{S^cS} - \bH_{S^cS}) (\bH_{SS})^{-1} \\
	T_3 &\triangleq (\bHhat_{S^cS} - \bH_{S^cS})((\bHhat_{SS})^{-1} - (\bH_{SS})^{-1}) \\
	T_4 &\triangleq \bH_{S^cS} (\bH_{SS})^{-1}   \, .
	\end{split}
	\end{align}
	Then it follows that  $\| \bHhat_{S^cS} (\bHhat_{SS})^{-1} \|_{\infty} \leq \| T_1 \|_{\infty} + \| T_2 \|_{\infty} + \| T_3 \|_{\infty} + \| T_4 \|_{\infty}$. Now, we will bound each term separately. First, recall that Assumption \ref{assum:mutual incoherence condition} ensures that $\| T_4 \|_{\infty} \leq 1 - \xi$.
	\paragraph{Controlling $T_1$.} We can rewrite $T_1$ as,
	\begin{align}
	\begin{split}
	T_1 = - \bH_{S^cS} (\bH_{SS})^{-1} (\bHhat_{SS} - \bH_{SS}) (\bHhat_{SS})^{-1}
	\end{split}
	\end{align}
	then,
	\begin{align}
	\begin{split}
	\| T_1 \|_{\infty} &= \| \bH_{S^cS} (\bH_{SS})^{-1} (\bHhat_{SS} - \bH_{SS}) (\bHhat_{SS})^{-1} \|_{\infty} \\
	&\leq \| \bH_{S^cS} (\bH_{SS})^{-1} \|_{\infty} \| (\bHhat_{SS} - \bH_{SS})\|_{\infty} \| (\bHhat_{SS})^{-1} \|_{\infty} \\
	&\leq (1 - \xi)  \| (\bHhat_{SS} - \bH_{SS})\|_{\infty} \sqrt{s} \| (\bHhat_{SS})^{-1} \|_2 \\
	&\leq (1 - \xi) \| (\bHhat_{SS} - \bH_{SS})\|_{\infty} \frac{2\sqrt{s}}{C_{\min}} \\
	&\leq \frac{\xi}{6} 
	\end{split}
	\end{align}
	The last inequality holds with probability at least $1 - 2\exp( - \frac{c C_{\min}^2 n}{4} + s) - 4s^2 \exp( - \frac{n C_{\min}^2 \xi^2 }{18432(1-\xi)^2 s^3(1 + 4\sigma^2) \max_{l} \Sigma_{ll}^2} )$ by taking $\delta = \frac{C_{\min} \xi}{12 (1 - \xi) \sqrt{s}}$.
	\paragraph{Controlling $T_2$.} Recall that $T_2 = (\bHhat_{S^cS} - \bH_{S^cS}) (\bH_{SS})^{-1}$. Thus,
	\begin{align}
	\begin{split}
	\| T_2 \|_{\infty} &\leq \sqrt{s} \| (\bH_{SS})^{-1} \|_2 \| (\bHhat_{S^cS} - \bH_{S^cS}) \|_{\infty} \\
	&\leq \frac{\sqrt{s}}{C_{\min}}   \| (\bHhat_{S^cS} - \bH_{S^cS}) \|_{\infty} \\
	&\leq \frac{\xi}{6}
	\end{split}
	\end{align}
	The last inequality holds with probability at least $1 - 4(d-s)s \exp( -\frac{n C_{\min}^2 \xi^2 }{4608 s^3 (1 + 4\sigma^2) \max_{l} \Sigma_{ll}^2 } )$ by choosing $\delta = \frac{C_{\min} \xi}{6 \sqrt{s}}$.
	\paragraph{Controlling $T_3$.} Note that,
	\begin{align}
	\begin{split}
	\| T_3 \|_{\infty} &\leq \| (\bHhat_{S^cS} - \bH_{S^cS}) \|_{\infty} \| ((\bHhat_{SS})^{-1} - (\bH_{SS})^{-1}) \|_{\infty} \\
	&\leq \frac{\xi}{6}
	\end{split}
	\end{align}
	The last inequality holds with probability at least $1 - 4(d-s)s \exp(- \frac{n \xi}{768s^2 (1 + 4\sigma^2) \max_{l} \Sigma_ll^2 }) - 2\exp(- \frac{c\xi C_{\min}^4n}{24 s} + s) - 2 \exp( - \frac{cC_{\min}^2n}{4} + s)$ by choosing $\delta = \sqrt{\frac{\xi}{6}}$ in the first and third inequality of equation \eqref{eq:helper mutual incoherence}. By combining all the above results, we prove Lemma \ref{lem:sample mutual incoherence condition}. The specific results for $n_1$ and $n_2$ follow directly.
\end{proof}

\section{Proof of Lemma~\ref{lem:bound on Delta}}
\label{proof:lem:bound on Delta}

\paragraph{Lemma ~\ref{lem:bound on Delta}.}
	\emph{	If Assumptions~\ref{assum:identifiability},~\ref{assum:postive definite} and \ref{assum:mutual incoherence condition} hold, and $\lambda_1 \geq 8 \rho \sigma_e \sqrt{n_1\log d}$, $\lambda_2 \geq 8 \rho \sigma_e \sqrt{n_2\log d}$, $n_1 = \Omega( \frac{s_1^3 \log d}{\tau(C_{\min}, \xi, \sigma, \Sigma)})$, and $n_2 = \Omega( \frac{s_2^3 \log d}{\tau(C_{\min}, \xi, \sigma, \Sigma)})$ then 
	$\| \Delta_1 \|_2 \leq (2 +  b_1)\frac{2\lambda_1 \sqrt{s_1}}{C_{\min} n_1} $ and $\| \Delta_2 \|_2 \leq (2 +  b_2)\frac{2\lambda_2 \sqrt{s_2}}{C_{\min} n_2} $ with probability at least $1 - \calO(\frac{1}{d})$ where $\tau(C_{\min}, \xi, \sigma, \Sigma)$ is a constant independent of $s_1, s_2, d, n_1$ or $n_2$.
}
\begin{proof}
	It suffices to prove the result for $\Delta_1$ as the result for $\Delta_2$ follows in the same way. Note,
	\begin{align*}
	\begin{split}
		 \tilde{\beta}_1 &= \arg\min\limits_{\beta \in \real^{s_1}} \sum_{i=1} \frac{t_i^* + 1}{2} (y_i - X_{i_P}^\T \beta )^2 + \lambda_1 (\| \beta \|_1 + 1)^2  \\
		 &= \arg\min\limits_{\beta \in \real^{s_1}} \sum_{i\in \calI_1}  (y_i - X_{i_P}^\T \beta )^2 + \lambda_1 (\| \beta \|_1 + 1)^2  \\
	\end{split}
	\end{align*}
	The optimal $\tilde{\beta}_1$ must satisfy stationarity KKT condition at the optimum, i.e.,
	\begin{align*}
	\begin{split}
		\sum_{i \in \calI_1} X_{i_P} (- y_i + X_{i_P}^\T \tilde{\beta}_1) + z \lambda_1 (z^\T \tilde{\beta}_1 + 1) = \bzero
	\end{split}
	\end{align*} 
	where $\| \tilde{\beta}_1 \|_1 = z^\T \tilde{\beta}_1$ and $z$ is in the subdifferential set of $\| \tilde{\beta}_1 \|_1$ and $\| z\|_{\infty} \leq 1$. Since $i \in \calI_1$, we can substitute $y_i = X_{i_P}^\T \beta_{1_P}^* + e_i $.
	\begin{align*}
	\begin{split}
		(\frac{1}{n_1}\sum_{i \in \calI_1} X_{i_P} X_{i_P}^\T + \frac{1}{n_1}\lambda_1 z z^\T) (\beta_{1_P}^* - \tilde{\beta}_1 ) + \frac{1}{n_1}(\sum_{i \in \calI_1} X_{i_P} e_i ) +  \frac{1}{n_1}\lambda_1 z z^\T \beta_{1_P}^* + \frac{1}{n_1}\lambda_1 z = \bzero
	\end{split}
	\end{align*}
	Note that $\widehat{H}_{1_{PP}} = \frac{1}{n_1}\sum_{i \in \calI_1} X_{i_P} X_{i_P}^\T$. Using norm-inequalities:
	\begin{align}
	\begin{split}
	\| \Delta_1 \|_2 \leq \| (\widehat{H}_{1_{PP}} + \frac{\lambda_1}{n_1} z z^\T)^{-1} \|_2 ( \| \frac{1}{n_1}(\sum_{i \in \calI_1} X_{i_P} e_i ) \|_2 +  \|\frac{1}{n_1}\lambda_1 z z^\T \beta_{1_P}^*\|_2 + \|\frac{1}{n_1}\lambda_1 z\|_2 )
	\end{split}
	\end{align}
	Using Lemma~\ref{lem:sample positive definite}, $\eig_{\min}(\widehat{H}_{1_{PP}}) \geq \frac{C_{\min}}{2}$, and using Weyl's inequality $\eig_{\min}(\widehat{H}_{1_{PP}} + \frac{\lambda_1}{n_1} z z^\T ) \geq \frac{C_{\min}}{2}$. It follows that $\| (\widehat{H}_{1_{PP}} + \frac{\lambda_1}{n_1} z z^\T)^{-1}  \|_2 \leq \frac{2}{C_{\min}}$.
	\begin{align}
	\begin{split}
	\| \Delta_1 \|_2 \leq \frac{2}{C_{\min}} ( \| \frac{1}{n_1}(\sum_{i \in \calI_1} X_{i_P} e_i ) \|_2 + \frac{\lambda_1 \sqrt{s_1}}{n_1} \| \beta_{1_P}^* \|_1 +  \frac{\lambda_1 \sqrt{s_1}}{n_1}  )
	\end{split}
	\end{align}
	We know that $\| \beta_1^* \|_1 \leq b_1$. Thus,
		\begin{align}
		\begin{split}
		\| \Delta_1 \|_2 \leq \frac{2}{C_{\min}} ( \| \frac{1}{n_1}(\sum_{i \in \calI_1} X_{i_P} e_i ) \|_2 + \frac{\lambda_1 \sqrt{s_1}}{n_1} b_1 +  \frac{\lambda_1 \sqrt{s_1}}{n_1}  )
		\end{split}
		\end{align}
	
	It only remains to bound $\| \frac{1}{n_1}(\sum_{i \in \calI_1} X_{i_P} e_i ) \|_2$ which we do in the following lemma.
	\begin{lemma}
		\label{lem:bound on l2 Xse}
If  $\lambda_1 \geq 8 \rho \sigma_e \sqrt{n_1\log d}$, then $\| \frac{1}{n_1} (\sum_{i \in \calI_1} X_{i_P} e_i ) \|_2 \leq \sqrt{s_1} \frac{\lambda_1}{n_1}$ with probability at least $1 - \calO(\frac{1}{d})$
	\end{lemma}
	Thus, it follows that
	\begin{align}
	\begin{split}
		\| \Delta_1 \|_\infty \leq \| \Delta_1 \|_2 \leq (2 +  b_1)\frac{2\lambda_1 \sqrt{s_1}}{C_{\min} n_1} 
	\end{split}
	\end{align} 
\end{proof}

\section{Proof of Lemma \ref{lem:bound on l2 Xse}}
\label{appndx:lem:bound on l2 Xse}
\paragraph{Lemma \ref{lem:bound on l2 Xse}}
\emph{		If  $\lambda_1 \geq 8 \rho \sigma_e \sqrt{n_1\log d}$, then $\| \frac{1}{n_1} (\sum_{i \in \calI_1} X_{i_P} e_i ) \|_2 \leq \sqrt{s_1} \frac{\lambda_1}{n_1}$ with probability at least $1 - \calO(\frac{1}{d})$.
}
\begin{proof}
	\label{proof:bound on l2 Xse}
We will start with $ \frac{1}{n_1} \sum_{i \in \calI_1} X_{i_P} e_i$. We take the $i$-th entry of $\frac{1}{n_1} \sum_{i \in \calI_1} X_{i_P} e_i$ for some $i \in P$, i.e., $| \frac{1}{n} \sum_{j \in \calI_1} X_{ji} \be_j |$. 		
Recall that $X_{ji}$ is a sub-Gaussian random variable with parameter $\rho^2$ and $e_j$ is a sub-Gaussian random variable with parameter $\sigma_e^2$. Then, $\frac{X_{ji}}{\rho}\frac{e_j}{\sigma_e}$ is a sub-exponential random variable with parameters $(4\sqrt{2}, 2)$. Using the concentration bounds for the sum of independent sub-exponential random variables~\citep{wainwright2019high}, we can write:
\begin{align}
\begin{split}
\prob( | \frac{1}{n_1} \sum_{j \in \calI_1} \frac{X_{ji}}{\rho}\frac{e_j}{\sigma_e} | \geq t) \leq 2 \exp(- \frac{n_1t^2}{64}), \; 0 \leq t \leq 8
\end{split}
\end{align} 
Taking a union bound across $i \in P$:
\begin{align}
\begin{split}
&\prob( \exists i \in P \mid | \frac{1}{n_1} \sum_{j \in \calI_1} \frac{X_{ji}}{\rho}\frac{e_j}{\sigma_e} | \geq t) \leq 2s_1 \exp(- \frac{n_1t^2}{64})\\
&0 \leq t \leq 8
\end{split}
\end{align}

	It follows that $\| \| \frac{1}{n_1} (\sum_{i \in \calI_1} X_{i_P} e_i ) \|_2 \|_2 \leq \sqrt{s} t$ with probability at least $1 - 2s \exp(- \frac{nt^2}{64 \rho^2\sigma_e^2})$ for some $0 \leq t \leq 8 \rho \sigma_e$. Taking $t = \frac{\lambda_1}{n}$, we get the desired result.
\end{proof}

\section{Proof of Lemma~\ref{lem:setting of mu_i and nu_i}}
\label{proof:lem:setting of mu_i and nu_i}

\paragraph{Lemma~\ref{lem:setting of mu_i and nu_i}}
\emph{If Assumptions~\ref{assum:identifiability},~\ref{assum:postive definite} and \ref{assum:mutual incoherence condition} hold, and $\lambda_1 \geq 8 \rho \sigma_e \sqrt{n_1\log d}$, $\lambda_2 \geq 8 \rho \sigma_e \sqrt{n_2\log d}$, $n_1 = \Omega( \frac{s_1^3 \log^2 d}{\tau(C_{\min}, \xi, \sigma, \Sigma)})$, and $n_2 = \Omega( \frac{s_2^3 \log^2 d}{\tau(C_{\min}, \xi, \sigma, \Sigma)})$ then $\mu_i \geq 0, \forall i \in \calI_1$ and $\nu_i \geq 0, \forall i \in \calI_2$ with probability at least $1 - \calO(\frac{1}{d})$ where $\tau(C_{\min}, \xi, \sigma, \Sigma)$ is a constant independent of $s_1, s_2, d, n_1$ or $n_2$.
}
\begin{proof}
	We start with the setting of $\mu_i$ when $i$ is in $\calI_1$.
	\begin{align}
	\begin{split}
	\mu_i &= - \frac{1}{2} \inner{\overline{S}_i^P}{\overline{W}} + \frac{1}{2}  \inner{\overline{S}_i^Q}{\overline{U}} \\
	&= - \frac{1}{2} (y_i - X_{i_P}^\T \tilde{\beta}_1)^2 + \frac{1}{2}  (y_i - X_{i_Q}^\T \tilde{\beta}_2)^2\\
	&= - \frac{1}{2} (y_i - X_{i_P}^\T \beta_{1_P}^* + X_{i_P}^\T (\tilde{\beta}_1 - \beta_{1_P}^*)  )^2 +  \frac{1}{2} (y_i - X_{i_Q}^\T \beta_{2_Q}^* + X_{i_Q}^\T (\tilde{\beta}_2 - \beta_{2_Q}^*)  )^2\\
	&= - \frac{1}{2} ( (y_i - X_{i_P}^\T \beta_{1_P}^*)^2 + (\tilde{\beta}_1 - \beta_{1_P}^*)^\T X_{i_P} X_{i_P}^\T (\tilde{\beta}_1 - \beta_{1_P}^*) + 2 (y_i - X_{i_P}^\T \beta_{1_P}^*)  X_{i_P}^\T (\tilde{\beta}_1 - \beta_{1_P}^*) ) +\\
	& \frac{1}{2} ( (y_i - X_{i_Q}^\T \beta_{2_Q}^*)^2 + (\tilde{\beta}_2 - \beta_{2_Q}^*)^\T X_{i_Q} X_{i_Q}^\T (\tilde{\beta}_2 - \beta_{2_Q}^*) + 2 (y_i - X_{i_Q}^\T \beta_{2_Q}^*)  X_{i_Q}^\T (\tilde{\beta}_2 - \beta_{2_P}^*) ) 
	\end{split}
	\end{align}
	Since $i \in \calI_1$, we can substitute $y_i = X_{i_P}^\T \beta_{1_P}^* + e_i$.
	\begin{align}
	\begin{split}
	\mu_i =& - \frac{1}{2} (y_i - X_{i_P}^\T \beta_{1_P}^*)^2 + \frac{1}{2} (y_i - X_{i_Q}^\T \beta_{2_Q}^*)^2 - \frac{1}{2} \Delta_1^\T X_{i_P} X_{i_P}^\T \Delta_1 + \frac{1}{2} \Delta_2^\T X_{i_Q} X_{i_Q}^\T \Delta_2  -  e_i X_{i_P}^\T \Delta_1 + \\
	&(X_{i_P}^\T \beta_{1_P}^* + e_i -  X_{i_Q}^\T \beta_{2_Q}^*  )X_{i_Q}^\T \Delta_2 \\
	=& - \frac{1}{2} (y_i - X_{i_P}^\T \beta_{1_P}^*)^2 + \frac{1}{2} (y_i - X_{i_Q}^\T \beta_{2_Q}^*)^2 - \frac{1}{2} \Delta_1^\T X_{i_P} X_{i_P}^\T \Delta_1 + \frac{1}{2} \Delta_2^\T X_{i_Q} X_{i_Q}^\T \Delta_2  -  e_i X_{i_P}^\T \Delta_1 +\\
	& e_i X_{i_Q}^\T \Delta_2 + (\beta_1^* - \beta_2^*)^\T X_i X_i^\T (\beta_2 - \beta_2^*)  
	\end{split}
	\end{align}
	Using bounds on the eigenvalue of data matrix, Assumption~\ref{assum:identifiability} and bounds on $\| \Delta_1 \|_2$ and $\| \Delta_2 \|_2$, we can place a bound on $\mu_i$.
	\begin{align}
	\begin{split}
	\mu_i \geq \epsilon - \frac{n_1 3C_{\max}}{2} \| \Delta_1 \|_2^2 + \frac{n_2 C_{\min}}{2} \| \Delta_2 \|_2^2 - | e_i X_{i_P}^\T   \Delta_1 |  - | e_i X_{i_Q}^\T  \Delta_2 | - \frac{n 3C_{\max}}{2} \| (\beta_1^* - \beta_2^*) \|_2 \| \Delta_2 \|_2  
	\end{split}
	\end{align}
	We still need bound to bound $| e_i X_{i_P}^\T   \Delta_1 | $ and $| e_i X_{i_Q}^\T   \Delta_2 | $ which we do in the following lemma.
	\begin{lemma}
		\label{lem:bound e_i X_i Delta}
	The following holds:
	\begin{enumerate}
		\item For fixed $\|\Delta_1\|_2$, $\prob( | e_i X_{i_P}^\T \Delta_1 | \leq \frac{\epsilon}{4})$ with probability at least $1 - \calO(\frac{1}{d})$.
		\item For fixed $\|\Delta_2\|_2$, $\prob( | e_i X_{i_Q}^\T \Delta_2 | \leq \frac{\epsilon}{4})$ with probability at least $1 - \calO(\frac{1}{d})$.
	\end{enumerate}
	\end{lemma} 
	\begin{proof}
 		Recall that $X_{i_P}^\T \Delta_1$ is a sub-Gaussian random variable with parameter $\rho^2 \| \Delta_1 \|_2^2$ and $e_i$ is a sub-Gaussian random variable with parameter $\sigma_e^2$.
 		Then, $\frac{X_{i_P}^\T \Delta_1}{\rho \| \Delta_1 \|_2} \frac{e_i}{\sigma_e}$ is a sub-exponential random variable with parameters $(4\sqrt{2}, 2)$. Using the concentration bounds for the sum of independent sub-exponential random variables~\citep{wainwright2019high}, we can write:
 		\begin{align}
 		\begin{split}
 		\prob( | \frac{X_{i_P}^\T \Delta_1}{\rho \| \Delta_1 \|_2}\frac{e_i}{\sigma_e} | \geq t) \leq 2 \exp(- \frac{t^2}{64}), \; 0 \leq t \leq 8
 		\end{split}
 		\end{align}  
 		Taking $t = \frac{t}{\rho \| \Delta_1 \|_2\sigma_e }$, we get
 		\begin{align}
 		\begin{split}
 		\prob( | e_i X_{i_P}^\T \Delta_1 | \geq t) \leq 2 \exp(- \frac{t^2}{64 \rho^2 \| \Delta_1 \|_2^2 \sigma_e^2}), \; 0 \leq t \leq 8\rho \| \Delta_1 \|_2\sigma_e
 		\end{split}
 		\end{align}
 		We take $t = \frac{\epsilon}{4}$, then
		\begin{align}
		\begin{split}
		\prob( | e_i X_{i_P}^\T \Delta_1 | \geq \frac{\epsilon}{4}) \leq 2 \exp(- \frac{\epsilon^2}{16 \times 64 \rho^2 \| \Delta_1 \|_2^2 \sigma_e^2}), \; 0 \leq \epsilon \leq 32\rho \| \Delta_1 \|_2\sigma_e
		\end{split}
		\end{align}
		Since  $\| \Delta_1 \|_2$ is upper bounded with $\calO( \frac{\lambda_1}{n_1} \sqrt{s_1} )$ and $n_1$ is of order $\calO(s_1^3 \log^2 d)$, thus $\prob( | e_i X_{i_P}^\T \Delta_1 | \leq \frac{\epsilon}{4})$ with probability at least $1 - \calO(\frac{1}{d})$. Similarly, $\prob( | e_i X_{i_Q}^\T \Delta_2 | \leq \frac{\epsilon}{4})$ with probability at least $1 - \calO(\frac{1}{d})$.
	\end{proof}
	Till now, we have considered $\| \Delta_1 \|_2$ and $\| \Delta_2 \|_2$ to be fixed quantity, however they are also upper bounded by $\calO( \frac{\lambda_1}{n_1} \sqrt{s_1} )$ with probability at least $1 - \calO(\frac{1}{d})$, thus the overall probability that $u_i \geq 0, i \in \calI_1$ is at least  $1 - \calO(\frac{1}{d})$ as long as 
	\begin{align}
	\begin{split}
		\epsilon \geq 3 n_1 C_{\max} \| \Delta_1 \|_2^2 - n_2 C_{\min} \| \Delta_2 \|_2^2 + 3 n C_{\max} \| \beta_1^* - \beta_2^* \|_2 \| \Delta_2 \|_2  
	\end{split}
	\end{align}
	We need to take a union bound across entries in $\calI_1$ which changes the probability to at least $1 - \calO(\exp( - \log d + \log n_1))$ which is still dominated by $1 - \calO(\frac{1}{d})$.
	
\end{proof}

\section{Proof of Lemma~\ref{lem:zero eigenvalue}}
\label{proof:lem:zero eigenvalue}

\paragraph{Lemma ~\ref{lem:zero eigenvalue}}
	\emph{Both $\Pi$ and $\Lambda$ have zero eigenvalues corresponding to eigenvectors $\begin{bmatrix} \tilde{\beta}_1 \\ 1 \end{bmatrix}$ and $\begin{bmatrix} \tilde{\beta}_2 \\ 1 \end{bmatrix}$ respectively.
}
\begin{proof}
		It suffices to prove the result for $\Pi$ as the result for $\Lambda$ follows in the same way. Note,
		\begin{align*}
		\begin{split}
		\tilde{\beta}_1 &= \arg\min\limits_{\beta \in \real^{s_1}} \sum_{i=1} \frac{t_i^* + 1}{2} (y_i - X_{i_P}^\T \beta )^2 + \lambda_1 (\| \beta \|_1 + 1)^2  \\
		&= \arg\min\limits_{\beta \in \real^{s_1}} \sum_{i\in \calI_1}  (y_i - X_{i_P}^\T \beta )^2 + \lambda_1 (\| \beta \|_1 + 1)^2  
		\end{split}
		\end{align*}
		The optimal $\tilde{\beta}_1$ must satisfy stationarity KKT condition at the optimum, i.e.,
		\begin{align*}
		\begin{split}
		\sum_{i \in \calI_1} X_{i_P} (- y_i + X_{i_P}^\T \tilde{\beta}_1) + z \lambda_1 (z^\T \tilde{\beta}_1 + 1) = \bzero
		\end{split}
		\end{align*} 
		
		By little algebraic manipulation, we can rewrite the above as following:
		\begin{align*}
		\begin{split}
		(\sum_{i=1}^n \frac{t_i^* + 1}{2} \bar{S}_i^P + \lambda_1 Z + I_{\alpha}) \begin{bmatrix} \tilde{\beta}_1 \\ 1 \end{bmatrix} = \bzero
		\end{split}
		\end{align*} 
		where $Z = \begin{bmatrix} z \\ 1 \end{bmatrix} \begin{bmatrix} z^\T & 1 \end{bmatrix} = \sign(\overline{W})$. Clearly,
		\begin{align*}
		\begin{split}
			\Pi \begin{bmatrix} \tilde{\beta}_1 \\ 1 \end{bmatrix} = \bzero
		\end{split}
		\end{align*}
		Similarly, we can show
		\begin{align*}
		\begin{split}
		\Lambda \begin{bmatrix} \tilde{\beta}_2 \\ 1 \end{bmatrix} = \bzero
		\end{split}
		\end{align*}
\end{proof}

\section{Proof of Lemma~\ref{lem:strictly positive second eigenvalue}}
\label{proof:lem:strictly positive second eigenvalue}

\paragraph{Lemma~\ref{lem:strictly positive second eigenvalue}}
\emph{If Assumption \ref{assum:postive definite} holds and $n_1 = \Omega(\frac{s_1 + \log d}{C_{\min}^2})$ and $n_2 = \Omega(\frac{s_2 + \log d}{C_{\min}^2})$, then the second eigenvalues of  $\Pi$ and $\Lambda$ are strictly positive with probability at least $1 - \calO(\frac{1}{d})$, i.e.,  $\eig_2(\Pi) > 0$ and $\eig_2(\Lambda) > 0$.
}
\begin{proof}
	It suffices to prove the result for $\Pi$ as similar arguments can be used to prove the result for $\Lambda$. We know
	\begin{align}
	\begin{split}
		\Pi &= \sum_{i \in \calI_1} S_i^P + \lambda_1 Z + I_{\alpha} \\
		&= \sum_{i \in \calI_1} \begin{bmatrix} X_i X_i^\T & - X_i y_i \\ - y_i X_i^\T & y_i^2 \end{bmatrix} + \lambda_1 \begin{bmatrix} z z^\T & z \\ z^\T & 1 \end{bmatrix} + I_\alpha \\
		&= \begin{bmatrix} \sum_{i \in \calI_1} X_i X_i^\T + \lambda_1 zz^\T & \sum_{i \in \calI_1} - X_i y_i + \lambda_1 z \\
		\sum_{i \in \calI_1} - y_i X_i^\T + \lambda_1 z^\T & \sum_{i \in \calI_1} y_i^2 + \lambda_1 + \alpha
		 \end{bmatrix}
	\end{split}
	\end{align}
	Also note that $\alpha =  - \inner{ \sum_{i \in \calI_1} S_i^P + \lambda_1 Z }{\overline{W}} = - \sum_{i\in \calI_1}  (y_i - X_{i_P}^\T \tilde{\beta}_1 )^2 + \lambda_1 (\| \tilde{\beta}_1 \|_1 + 1)^2   $. We also know that $\tilde{\beta}_1$ satisfies the stationarity KKT condition, i.e.,
		\begin{align*}
		\begin{split}
		\sum_{i \in \calI_1} X_{i_P} (- y_i + X_{i_P}^\T \tilde{\beta}_1) + z \lambda_1 (z^\T \tilde{\beta}_1 + 1) = \bzero \\
		\tilde{\beta}_1  = - ( \sum_{i \in \calI_1} X_i X_i^\T + \lambda_1 zz^\T)^{-1} (\sum_{i \in \calI_1} - X_i y_i + \lambda_1 z)
		\end{split}
		\end{align*} 
	Using the stationarity KKT condition, we can simplify objective function value of optimization problem~\eqref{eq:opt prob invex mlr compact} at $\tilde{\beta}_1$ to $  \sum_{i \in \calI_1} y_i^2 + (\sum_{i \in \calI_1} - y_i X_i^\T + \lambda_1 z^\T) \tilde{\beta}_1 + \lambda_1$. 	Now, we invoke Haynesworth's inertia additivity formula~\citep{haynsworth1968determination} to prove our claim. Let $R$ be a block matrix of the form $ R = \begin{bmatrix} A & B \\ B^\T & C \end{bmatrix}$, then inertia of matrix $R$, denoted by $\inertia(R)$, is defined as the tuple $(|\eig_{+}(R)|, |\eig_{-}(R)|, |\eig_0(R)| )$ where $|\eig_+(R)|$ is the number of positive eigenvalues, $|\eig_-(R)|$ is the number of negative eigenvalues and $|\eig_0(R)|$ is the number of zero eigenvalues of matrix $R$. Haynesworth's inertia additivity formula is given as:
	\begin{align}
	\label{eq:Haynesworth inertia additivity formula}
	\begin{split}
	\inertia(R) = \inertia(A) + \inertia(C - B^\T A^{-1} B)
	\end{split}
	\end{align}
	We take $A =  \sum_{i \in \calI_1} X_i X_i^\T + \lambda_1 zz^\T  $, $B = \sum_{i \in \calI_1} - X_i y_i + \lambda_1 z $ and $C = \sum_{i \in \calI_1} y_i^2 + \lambda_1 + \alpha$. It should be noted that $C - B^\T A^{-1} B$ evaluates to zero. Thus,
	\begin{align}
	\begin{split}
	\inertia(\Pi) = \inertia(\sum_{i \in \calI_1} X_i X_i^\T + \lambda_1 zz^\T ) + \inertia(0)
	\end{split}
	\end{align} 
	We note that $0$ has precisely one zero eigenvalue and no other eigenvalues. Moreover, from Lemma~\ref{lem:sample positive definite} and Weyl's inequality:
	\begin{align}
	\begin{split}
		\eig_{\min}(\sum_{i \in \calI_1} X_i X_i^\T + \lambda_1 zz^\T) \geq \frac{C_{\min}}{2} > 0
	\end{split}
	\end{align} 
	with probability at least $1 - \calO(\frac{1}{d})$ as long as $n_1 = \Omega(\frac{s_1 + \log d}{C_{\min}^2})$. It follows that the second eigenvalue of $\Pi$ is strictly positive. Similar, arguments can be made for $\Lambda$. 
\end{proof}

\section{Proof of Lemma~\ref{lem:bound X_se and X_s^ce}}
\label{proof:lem:bound X_se and X_s^ce}

\paragraph{Lemma~\ref{lem:bound X_se and X_s^ce}}
\emph{Let $\lambda_1 \geq \frac{64 \rho \sigma_e}{\xi} \sqrt{n_1\log d}$. Then the following holds true:
	\begin{align*}
	\begin{split}
	&\prob(\| \frac{1}{\bar{\lambda}_1 } \frac{1}{n_1} \sum_{i \in \calI_1} X_{i_P} e_i \|_{\infty} \geq \frac{\xi}{8 - 4\xi}) \leq \calO(\frac{1}{d}),\;\;\; \prob(\| \frac{1}{\bar{\lambda}_1} \frac{1}{n_1} \sum_{i \in \calI_1} X_{i_{P^c}} e_i \|_{\infty} \geq \frac{\xi}{8}) \leq \calO(\frac{1}{d}) 
	\end{split}
	\end{align*}
}
\begin{proof}
		\label{proof:bound X_se and X_s^ce}
		We will start with $ \frac{1}{n_1} \sum_{i \in \calI_1} X_{i_P} e_i$. We take the $i$-th entry of $\frac{1}{n_1} \sum_{i \in \calI_1} X_{i_P} e_i$ for some $i \in P$, i.e., $| \frac{1}{n} \sum_{j \in \calI_1} X_{ji} \be_j |$. 		
		Recall that $X_{ji}$ is a sub-Gaussian random variable with parameter $\rho^2$ and $e_j$ is a sub-Gaussian random variable with parameter $\sigma_e^2$. Then, $\frac{X_{ji}}{\rho}\frac{e_j}{\sigma_e}$ is a sub-exponential random variable with parameters $(4\sqrt{2}, 2)$. Using the concentration bounds for the sum of independent sub-exponential random variables~\citep{wainwright2019high}, we can write:
		\begin{align}
		\begin{split}
		\prob( | \frac{1}{n_1} \sum_{j \in \calI_1} \frac{X_{ji}}{\rho}\frac{e_j}{\sigma_e} | \geq t) \leq 2 \exp(- \frac{n_1t^2}{64}), \; 0 \leq t \leq 8
		\end{split}
		\end{align} 
		Taking a union bound across $i \in P$:
		\begin{align}
		\begin{split}
		&\prob( \exists i \in P \mid | \frac{1}{n_1} \sum_{j \in \calI_1} \frac{X_{ji}}{\rho}\frac{e_j}{\sigma_e} | \geq t) \leq 2s_1 \exp(- \frac{n_1t^2}{64})\\
		&0 \leq t \leq 8
		\end{split}
		\end{align}
		
		Taking $t = \frac{\bar{\lambda}_1 t}{ \rho \sigma_e}$, we get: 
		\begin{align}
		\begin{split}
		&\prob( \exists i \in P \mid |\frac{1}{\bar{\lambda}_1} \frac{1}{n_1} \sum_{j=1}^n X_{ji} e_j | \geq t) \leq 2s_1 \exp(- \frac{n_1 \bar{\lambda}_1^2 t^2}{64 \rho^2 \sigma_e^2})\\
		&0 \leq t \leq 8 \frac{\rho \sigma_e}{\bar{\lambda}_1}
		\end{split}
		\end{align}
		
		It follows that $\| \frac{1}{\bar{\lambda}_1 } \frac{1}{n_1} \sum_{i \in \calI_1} X_{i_P} e_i \|_{\infty} \leq  t$ with probability at least $1 - 2s_1 \exp(- \frac{n \bar{\lambda}_1^2 t^2}{64 \rho^2 \sigma_e^2})$.
		
		Using a similar argument, we can show that $\| \frac{1}{\bar{\lambda}_1} \frac{1}{n_1} \sum_{i \in \calI_1} X_{i_{P^c}} e_i \|_{\infty} \leq  t$ with probability at least $1 - 2(d - s_1) \exp(- \frac{n_1 \bar{\lambda}_1^2 t^2}{64 \rho^2 \sigma_e^2})$. Taking $t= \frac{\xi}{8 - 4 \xi}$ and $\frac{\xi}{8}$ in the first and second inequality of Lemma \ref{lem:bound X_se and X_s^ce} and choosing the provided setting of $\lambda_1$ and $n_1$ completes our proof.
\end{proof}


\end{document}